\documentclass[12pt]{article}
\usepackage[top=1in, bottom=1in, left=1in, right=1in]{geometry}
\usepackage[utf8]{inputenc}
\usepackage{amsmath,amsthm,amsfonts,amssymb}
\usepackage{url}
\usepackage{comment}
\usepackage{hyperref}
\usepackage[capitalise]{cleveref}
\usepackage{mathrsfs}
\usepackage{xcolor}
\usepackage[backend=biber,style=alphabetic,maxbibnames=99]{biblatex}
\title{Efficiently learning and sampling multimodal distributions with data-based initialization}
\newcommand{\norm}[1]{\left\Vert {#1}\right\Vert}
\newcommand{\Renyi}{\mathcal{R}}
\newcommand{\E}{\mathbb{E}}
\newcommand{\set}[1]{\left\{ #1\right\}}

\renewcommand{\Pr}[1]
{\operatorname{Pr}\left(#1\right)}
\newtheorem{theorem}{Theorem}
\newtheorem*{theorem*}{Theorem}
\newtheorem{lemma}[theorem]{Lemma}
\newtheorem*{lemma*}{Lemma}
\newtheorem{prop}[theorem]{Proposition}
\newtheorem*{prop*}{Proposition}
\newtheorem{corollary}[theorem]{Corollary}
\newtheorem{remark}[theorem]{Remark}
\newtheorem{assumption}{Assumption}
\newtheorem{definition}[theorem]{Definition}

\newcommand{\sL}[0]{\mathscr{L}}

\newcommand{\N}[0]{\mathbb{N}}

\newcommand{\Pj}[0]{\mathbb{P}}

\newcommand{\R}[0]{\mathbb{R}}

\newcommand{\al}[0]{\alpha}
\newcommand{\be}[0]{\beta}
\newcommand{\ga}[0]{\gamma}
\newcommand{\Ga}[0]{\Gamma}
\newcommand{\de}[0]{\delta}

\newcommand{\ep}[0]{\varepsilon}

\newcommand{\ka}[0]{\kappa}
\newcommand{\la}[0]{\lambda}

\newcommand{\Te}[0]{\Theta}

\newcommand{\Om}[0]{\Omega}
\newcommand{\si}[0]{\sigma}

\newcommand{\bigot}[0]{\bigotimes}

\newcommand{\subeq}[0]{\subseteq}

\newcommand{\bs}[0]{\backslash}
\newcommand{\iy}[0]{\infty}

\newcommand{\na}[0]{\nabla}

\newcommand{\rc}[1]{\frac{1}{#1}}
\newcommand{\prc}[1]{\pa{\rc{#1}}}

\newcommand{\fc}[2]{\frac{#1}{#2}}
\newcommand{\sfc}[2]{\sqrt{\frac{#1}{#2}}}
\newcommand{\pf}[2]{\pa{\frac{#1}{#2}}}

\newcommand{\dd}[2]{\frac{d #1}{d #2}}

\newcommand{\nb}[0]{\nabla}

\newcommand{\dx}{\,dx}

\newcommand{\ab}[1]{\left| {#1} \right|}
\newcommand{\an}[1]{\left\langle {#1}\right\rangle}
\newcommand{\ba}[1]{\left[ {#1} \right]}
\newcommand{\bc}[1]{\left\{ {#1} \right\}}

\newcommand{\fl}[1]{\left\lfloor {#1}\right\rfloor}

\newcommand{\pa}[1]{\left( {#1} \right)}
\newcommand{\pat}[1]{\left( \text{#1} \right)}
\newcommand{\ve}[1]{\left\Vert {#1}\right\Vert}

\newcommand{\vvt}[1]{{#1}{#1}^\top}

\newcommand{\vtv}[1]{{#1}^\top{#1}}

\newcommand{\ol}[1]{\overline{#1}}

\newcommand{\ub}[2]{\underbrace{#1}_{#2}}

\newcommand{\wt}[1]{\widetilde{#1}}

\newcommand{\amin}{\operatorname{argmin}}

\newcommand{\Ent}{\operatorname{Ent}}

\newcommand{\im}[0]{\operatorname{im}}

\newcommand{\KL}[0]{\operatorname{KL}}

\newcommand{\poly}{\operatorname{poly}}

\newcommand{\sgn}{\operatorname{sign}}

\newcommand{\spn}{\operatorname{span}}

\newcommand{\Tr}[0]{\operatorname{Tr}}
\newcommand{\Trace}[0]{\operatorname{Tr}}

\newcommand{\Var}[0]{\operatorname{Var}}

\providecommand{\cal}[1]{\mathcal{#1}}
\renewcommand{\cal}[1]{\mathcal{#1}}

\newcommand{\pull}[9]{
#1\ar@/_/[ddr]_{#2} \ar@{.>}[rd]^{#3} \ar@/^/[rrd]^{#4} & &\\
& #5\ar[r]^{#6}\ar[d]^{#8} &#7\ar[d]^{#9} \\}

\newcommand{\cmp}[9]{
\xymatrix{
#1 \ar[r]^{#4}{#5} \ar@/_2pc/[rr]^{#8}_{#9} & #2 \ar[r]^{#6}_{#7} & #3
}
}

\newcommand{\ha}[1]{\ar@{^(->}[#1]}
\newcommand{\ls}[1]{\ar@{-}[#1]}
\newcommand{\sj}[1]{\ar@{->>}[#1]}
\newcommand{\aq}[1]{\ar@{=}[#1]}
\newcommand{\acir}[1]{\ar@{}[#1]|-{\textstyle{\circlearrowright}}}
\newcommand{\acil}[1]{\ar@{}[#1]|-{\textstyle{\circlearrowleft}}}
\newcommand{\ard}[1]{\ar@{.>}[#1]}
\newcommand{\mt}[1]{\ar@{|->}[#1]}
\newcommand{\inm}[1]{\ar@{}[#1]|-{\in}}
\newcommand{\inr}{\ar@{}[d]|-{\rotatebox[origin=c]{-90}{$\in$}}}
\newcommand{\inl}{\ar@{}[u]|-{\rotatebox[origin=c]{90}{$\in$}}}

\newcommand{\maxr}[2]{\max_{\scriptsize \begin{array}{c}{#1}\\{#2}\end{array}}}

\newcommand{\sumo}[2]{\sum_{#1=1}^{#2}}
\newcommand{\sumz}[2]{\sum_{#1=0}^{#2}}

\newcommand{\beq}[1]{\begin{equation}\llabel{#1}}
\newcommand{\eeq}[0]{\end{equation}}
\newcommand{\bal}[0]{\begin{align*}}
\newcommand{\eal}[0]{\end{align*}}
\newcommand{\ban}[0]{\begin{align}}
\newcommand{\ean}[0]{\end{align}}

\newcommand{\fixme}[1]{{\color{red}#1}}
\newcommand{\llabel}[1]{\label{#1}\text{\fixme{\tiny#1}}}

\newcommand{\arxiv}[1]{\url{http://www.arxiv.org/abs/#1}}

\newcommand{\vocab}[1]{\textbf{#1}} 

\newcommand{\CP}[0]{C_{\textup P}}

\newcommand{\CLS}[0]{C_{\textup{LS}}}

\newcommand{\PI}[0]{\mathsf{PI}}
\newcommand{\LSI}[0]{\mathsf{LSI}}

\newcommand{\dL}[3]{X^{#1}_{#2,#3}}
\newcommand{\cL}[3]{\bar X^{#1}_{#2,#3}}
\newcommand{\esc}[0]{\varepsilon_{\textup{score}}}
\newcommand{\mode}[0]{x^*}
\newcommand{\mean}[0]{\overline x}
\newcommand{\Ombd}[0]{\Omega_{\textup{bd}}}
\newcommand{\etv}[0]{\varepsilon_{\textup{TV}}}
\newcommand{\esm}[0]{\varepsilon_{\textup{small}}}
\newcommand{\TV}[0]{\operatorname{TV}}
\allowdisplaybreaks[2]

\DeclareFontFamily{U}{wncy}{}
    \DeclareFontShape{U}{wncy}{m}{n}{<->wncyr10}{}
    \DeclareSymbolFont{mcy}{U}{wncy}{m}{n}
    \DeclareMathSymbol{\Sh}{\mathord}{mcy}{"58} 
    
\newcommand{\holden}[1]{\textcolor{purple}{[Holden: #1]}}

\newcommand{\hlnote}[1]{}
\newcommand{\junenote}[1]{}
\newcommand{\frednote}[1]{}
\renewcommand{\holden}[1]{}

\author{Frederic Koehler\thanks{University of Chicago} \and Holden Lee\thanks{Johns Hopkins University} \and Thuy-Duong Vuong\thanks{University of California, Berkeley}}

\newcommand{\D}{\mathcal{D}}

\renewcommand{\P}{\mathbb{P}}

\begin{document}

\maketitle

\begin{abstract}
We consider the problem of sampling a multimodal distribution with a Markov chain given a small number of samples from the stationary measure.
Although mixing can be arbitrarily slow, we show that if the Markov chain has a $k$th order spectral gap, 
initialization from a set of 
$\tilde O(k/\varepsilon^2)$ samples from the stationary distribution will, with high probability over the samples, efficiently generate a sample whose conditional law is $\varepsilon$-close in TV distance to the stationary measure.
In particular, this applies to mixtures of $k$ distributions satisfying a Poincar\'e inequality, with faster convergence when they satisfy a log-Sobolev inequality.
Our bounds are stable to perturbations to the Markov chain, and in particular work for Langevin diffusion over $\mathbb R^d$ with score estimation error, as well as Glauber dynamics combined with approximation error from pseudolikelihood estimation. 
This justifies the success of data-based initialization for score matching methods despite slow mixing for the data distribution, and improves and generalizes the results of \cite{koehler2023sampling} to have linear, rather than exponential, dependence on $k$ and apply to arbitrary semigroups. As a consequence of our results, we show for the first time that a natural class of low-complexity Ising measures can be efficiently learned from samples.
\end{abstract}

\newpage
\tableofcontents
\newpage

\section{Introduction}
Since its introduction in 1953 by \textcite{metropolis1953equation},
Markov-Chain Monte Carlo (MCMC) has become the one of the dominant approaches to sampling and integration of high-dimensional distributions in Bayesian statistics, computational physics, biostatistics, astronomy, machine learning, and many other areas. Typically, in MCMC we sample from a distribution of interest by simulating a Markov chain which converges to the correct stationary measure. 

One of the key mathematical questions concerning a Markov chain is its \emph{mixing time}---how quickly does the process forget its initialization and converge to stationarity? 
While some Markov chains rapidly mix to their stationary distributions, it is well-known that in many other cases, the existence of ``bottlenecks'' (i.e., sparse cuts) between modes leads to slow or ``torpid'' mixing, oftentimes exponential in the dimension of the problem. In general, whenever the distribution of interest is supported on multiple well-separated clusters or modes, standard MCMC methods like Metropolis-Hastings, Langevin dyamics, Glauber dynamics, and so on which make ``local'' moves will get stuck in the first cluster they reach. This phenomena is often referred to as metastability in the literature \cite{gayrard2004metastability,gayrard2005metastability}. 

A large body of research in MCMC, both in theory and practice, is on developing ways to overcome this difficulty. For example, popular methods such as simulated tempering \cite{marinari1992simulated} and parallel tempering \cite{swendsen1986replica} attempt to improve connectivity by varying the temperature of the system. In other situations, alternative Markov chains can be constructed which are able to cross between the modes --- for example, the celebrated Swendsen-Wang dynamics \cite{swendsen1987nonuniversal} which are provably able to sample from the ferromagnetic Ising model at all temperatures in polynomial time \cite{jerrum1993polynomial,guo2017random}. In other cases, the sampling problem is provably hard (see e.g. \cite{sly2012computational,galanis2016inapproximability}) so no computationally efficient Markov chain could possibly mix rapidly to the stationary measure unless $\mathsf P = \mathsf{NP}$. 

Hence, sampling can be computationally intractable in general. However, many cases where standard Markov chains are known to fail actually correspond to very simple distributions. For example, even a simple mixture of two well-separated Gaussians leads to exponentially large mixing time for the Langevin dynamics (in the separation distance).
Previous work in MCMC theory has studied in depth some of these failure cases (e.g. in the Curie-Weiss model, see related work below)
and developed specialized solutions to resolve the mixing time issue in a particular setting. We might hope that a more general approach can resolve the difficulty with multimodality for a large class of models.

\paragraph{Our contribution.} 
In this work, we develop general tools to analyze MCMC chains in multimodal situations. Analogous to the role of the spectral gap in the unimodal setting, the key mathematical object in our theory is the \emph{higher-order spectral gap} of the transition matrix or generator of a Markov process. Looking at such a notion of gap is very natural from the perspective of higher-order analogues of Cheeger's inequality \cite{lee2014multiway,louis2012many,gharan2014partitioning,Miclo2014OnHA}. These results roughly tell us that if the vertices of a graph can be separated into a small number of well-connected parts whose boundaries are sparse cuts, then there has to be a corresponding gap in the spectrum of the Laplacian after a small number of eigenvalues, and vice versa. 

When we have only a higher-order, rather than standard, spectral gap, we cannot hope for rapid mixing of the dynamics from an arbitrary initialization. However, in some applications there is a natural candidate for a warm start for the dynamics.
In particular, in the application of density estimation or \emph{generative modeling}, a distribution is 
learned from access to samples from the ground truth distribution. In these settings, samples from the ground truth are available, which naturally suggests the idea of \emph{data-based initialization} --- starting the dynamics from the empirical measure. The idea of data-based initialization has  appeared in the empirical machine learning literature in many different forms, for example as a part of the mechanics of ``contrastive divergence'' training for energy-based methods and other approximations to Maximum Likelihood Estimation. For a few related references in the empirical literature, see \cite{hinton2002training,xie2016theory,gao2018learning,nijkamp2019learning,nijkamp2020anatomy,wenliang2019learning}, and see also \cite{koehler2023sampling} for more discussion. In particular, the terminology of ``data-based initialization'' is as used in \cite{nijkamp2020anatomy}. 

Intuitively, if the underlying distribution is a mixture distribution, then each mixture component will have a roughly proportional representation in the samples from the empirical measure, so we might hope that the dynamics run for a polynomial amount of time can actually recover the ground truth distribution. A direct analysis along these lines was done in \cite{koehler2023sampling} in the case of a mixture of strongly log-concave distributions; however, handling the behavior of overlapping clusters complicated the analysis and ultimately led to a poor (exponential) dependence on the number of clusters or mixture components in the distribution. 

From here on, we revisit the analysis of data-based initialization from the spectral perspective. This yields a much more elegant proof with dramatically improved quantitative dependencies. Our approach applies to general Markov semigroups, and in particular lets us prove new results for the Glauber dynamics as well as the Langevin dynamics. We also can easily obtain natural extensions of our results, such as more rapid mixing under a component-wise log-Sobolev inequality. Crucially for density estimation applications, we also show that our results are robust to perturbations in the Markov chain, which is very important when the chain transitions are themselves estimated from data. The quantitative improvements in this theory are key to an illustrative new application---a new result for learning a class of Ising models well beyond the regime where previous approaches were known to succeed. 


\subsection{Main results}
The heart of this work in the following theorem which holds in a very general setting---it applies to all Markov semigroups, and shows that a higher-order spectral gap implies rapid mixing from data-based initialization. 
\begin{theorem*}[\Cref{thm:main}, simplified]
    Let $P_t = e^{t \sL}$ be a reversible Markov semigroup with self-adjoint generator $\sL$ and stationary measure $\pi$ defined over $\mathcal D$. Suppose that the generator satisfies $\la_{k+1}(-\sL)\ge \al$ and there are constants $t_0, R$ such that 
    \[
\pat{warm start after time $t_0$}\quad 
\forall y\in \mathcal D, \quad \chi^2(\de_yP_{t_0} \|\pi)\le R.
    \]
    Let $\mu_0 = \rc n \sumo jn \de_{Y_j}$ where $Y_1,\ldots,Y_n \sim \pi$ are independent samples, and define $\mu_t = \mu_0 P_t$.
        Then with probability $\ge 1-\de$, for $n=\Om\pa{\fc{k}{\etv^2} \ln \pf k\de}$ and $t\ge t_0+\rc \al \ln \pf{4R}{\etv^2}$, we have 
\[ \TV(\mu_t , \pi)\le \etv. \]
\end{theorem*}
In fact, we only need a weaker version of the warm-start condition, which is important in some applications---see \cref{thm:main} and \Cref{thm:generalize main} for the more general result. A higher-order spectral gap is easy to show for Glauber or Langevin dynamics if our distribution is a mixture of well-connected components---see \cref{l:egap}.

\paragraph{Langevin dynamics on mixtures.} Our first application is to Langevin dynamics on mixture distributions. In this application, we include a perturbation analysis which shows that an $L_2$ approximate score function suffices for sampling---this is important if the score function is learned from data via score matching (see e.g.\ \cite{hyvarinen2005estimation}). 
\begin{theorem*}[Langevin with score matching, \Cref{l:ld}, \Cref{t:ld-better}, simplified]
    Assume $\pi=\sumo ik p_i\pi_i$, where $\pi_i$ are $O(1)$-smooth and the means of the $\pi_i$ are at distance $\lesssim \sqrt d$. Suppose that the approximate score function $s$ satisfies 
\[ \E_{\pi}\ve{s - \nabla \ln \pi }^2\le \esc^2. \]
    Let $\mu_0 = \rc n \sumo jn \de_{Y_j}$, where $Y_1,\ldots,Y_n \sim \pi$ are independent samples, and suppose 
$n=\Om\pf{k\ln (k/\delta)}{\etv^2}$.
Let $(\bar{X}_t)_t$ be the continuous Langevin diffusion wrt $\pi$ initialized at $ \mu_0$ and driven by $s$, so it satisfies the SDE
\[ d\bar{X}_t = s(\bar{X}_t) dt + \sqrt{2}\, dW_t \]
for an independent Brownian motion $W_t$, and $\bar{X}_0\sim \mu_0.$ Let $\mu_t$ be the law of $\bar{X}_t$ conditioned on the empirical samples $Y_1, \dots, Y_n$, i.e., $\mu_t = \mathcal{L}(\bar{X}_t | Y_1, \dots, Y_n).$

Suppose that either: 
\begin{enumerate}
    \item Each $\pi_i$ satisfies a Poincar\'e inequality with constant $\le\rc{\al}$, and  
    $T = \Om \pa{\rc{\al}\pa{d + \ln \pf{k}{\etv}^2}}.$
    \item Each $\pi_i$ satisfies a log-Sobolev inequality with constant $\le \rc{\al}$, and 
    $T = \Om \pa{\fc{1}{\al}\ln \pf{dk}{\etv}}.$
\end{enumerate}
Then with probability at least $1-\delta$ over the randomness of $Y_1,\ldots,Y_n$,
\[
\TV(\mu_T , \pi) \le \sqrt{T}\esc  + \etv
\]
\end{theorem*}
The above result is stated for the continuous-time Langevin diffusion, but we also prove a version of this result for its discrete-time analogue, Langevin Monte Carlo \Cref{thm:langevin discretization}. We note there are previous results that quantify mixing up to multimodality in the distribution (e.g., \cite{balasubramanian2022towards} for convergence in Fisher information for averaged LMC, or \cite{tzen2018local,zhang2017hitting}); our innovation is to show that data-based initialization can lead to the much stronger condition of mixing. This is analogous to the difference between finding a stationary point vs.\ a global optimum in nonconvex optimization.
\begin{remark}[Matching sample complexity lower bound]
It is a classical fact that $\Theta(k/\etv^2)$ samples are needed to learn a distribution on the alphabet $\{1,\ldots,k\}$ within total variation distance $\etv$ (see e.g.\ \cite{han2015minimax} for references). This is a special case of our problem: it corresponds to the case where $\pi$ is a mixture of $k$ known components with disjoint support where only the mixing weights are unknown. (Note that the score function in this case does not depend on the mixing weights.) Thus, the dependence on $n$ in both of the previous theorems is optimal up to the $\log$ factor in $k$. 
\end{remark}
\begin{remark}[Using score matching in Gaussian mixture models]
Recent works \cite{chen2024learning,gatmiry2024learning} show how to learn a mixture of well-conditioned Gaussians from samples using a computationally efficient score matching approach. The key step is to show that the score function of such a distribution can be well-approximated using a piecewise-polynomial function, which can be efficiently estimated from data. Because mixtures of Gaussians are closed under convolution with noise, this implies that they can use the learned score functions at different noise levels to approximately sample via a denoising diffusion process (see e.g.\ \cite{benton2023linear,chen2023improved} and references therein). Since they show that the score function can be accurately estimated from samples (at least for a slightly noised version of the distribution, which is close in TV---see Proposition 2.1 and the proof of Theorem 4.1 in \cite{chen2023improved}), their score function estimate could be combined with our method (data-based initialization of the vanilla Langevin dynamics) to give an alternative and arguably simpler algorithm. 
\end{remark}

\paragraph{Glauber dynamics and an application to learning.} A similar result to \cref{t:ld-better} holds for Glauber dynamics (see \cref{lem:gibbs-approximate})---in this context, we also show that the dynamics are robust to a small error in the KL divergence, which would occur if we estimate the dynamics via pseudolikelihood estimation \cite{besag1975statistical}. 
Pseudolikelihood is a very classical and popular method, and the exact analogue of score matching for Glauber---see \cite{hyvarinen2007connections,koehler2022statistical} for more discussion.

As a concrete end-to-end application of the result for Glauber dynamics, we prove a new theorem about learning a large class of Ising models: those which are in some sense low complexity or approximately low rank. This class of models has been extensively studied in probability theory due to its connection to mean-field approximation---see the discussion of related work below.
As we discuss therein, this class of models is well outside of the realm where previous learning results (e.g.\ \cite{wu2019sparse,gaitonde2024unified}) can be applied.
\begin{theorem*}[Learning approximate low-rank Ising models, \Cref{t:learn-ising}, simplified]
Suppose $\pi$ is an Ising model, i.e. a probability measure on $\{ \pm 1\}^n$ satisfying
\[ \pi(x) \propto \exp\left(\rc 2\langle x, J x \rangle + \langle h, x \rangle\right) \]
for some symmetric interaction matrix $J\in \R^{n\times n}$ and external field vector $h\in \R^n$.

    Suppose $J$ has eigenvalues (ordered from the largest)
    \[ \la_1\ge \cdots \la_r> 1-\rc c\ge \la_{r+1}\ge\cdots \ge \la_n, \] 
    for a constant $c > 1$ and $-\sum_{j:\la_j<0}\la_j=O(1)$. 
        Given $(nr\la_1)^{O(r)}/\etv^4$ samples from $\pi_{J,h}$, with high probability over the samples, the distribution output pseudolikelihood estimation and Glauber dynamics from data-based initialization is within TV error $\etv$ of $\pi_{J,h}$.
\end{theorem*}
See the full statement of the theorem for more details and the precise definition of pseudolikelihood estimation.
\subsection{Other related work}
\paragraph{Learning Markov random fields from samples.} There have been too many works on learning Markov random fields from samples to give an exhaustive list, so we instead summarize some of the most recent and directly relevant works. 
Information-theoretically, it is known that density estimation of Ising models in TV distance can be done with polynomial dependence on the dimension $n$ and target accuracy $\epsilon$ \cite{devroye2020minimax}. However, all existing results for learning Ising models with computationally efficient algorithms, which typically use some variant of pseudolikelihood estimation and include works such as
\cite{ravikumar2010high,lokhov2018optimal,bresler2017learning,klivans2017learning,wu2019sparse,gaitonde2024unified}, either restrict the model to be in a high-temperature regime, or have an exponential dependence on some parameter in the sample complexity---typically the maximum $\ell_1$-norm of any row of the interaction matrix. In many examples this $\ell_1$-norm is polynomial in the dimension (see e.g.\ \cite{anari2024universality,gaitonde2024unified} for more discussion), and it is always linear in the ``inverse temperature'' $\beta$ of the system. So in particular, all of these results have an exponential dependence on $\beta$. In contrast, our result has polynomial dependence on these parameters when the interaction matrix is (approximately) low rank. See Remark~\ref{remark:separation} for a much more detailed discussion of the limitations of previous techniques.




\paragraph{Tempering and annealing on multimodal distributions.} 
The empirical community has developed many algorithms for sampling from multimodal distributions. This includes \emph{tempering} methods such as simulated tempering \cite{marinari1992simulated} and parallel tempering \cite{swendsen1986replica}, which involve constructing a Markov chain which varies the temperature of the system, as well as \emph{annealing} or \emph{sequential} methods such as sequential Monte Carlo \cite{liu1998sequential} and annealed importance sampling \cite{neal2001annealed}, which vary the temperature unidirectionally over time.

Efficient theoretical guarantees are known only in special cases; a necessary condition for all known results is that the distribution is decomposable into parts which are not too imbalanced (do not have a bottleneck) between different temperatures. 
For simulated tempering, these results can be used to show that simulated tempering with Langevin dynamics can sample from an isotropic mixture of Gaussians (or more generally, translates of a fixed log-concave distribution) \cite{ge2018beyond,ge2018simulated}. 
A version of simulated tempering was also used in \cite{koehler2022sampling} to sample from multimodal Ising models.
There are analogous results for sampling from multimodal distributions under stronger assumptions for parallel tempering \cite{woodard2009conditions,lee2023improved} and sequential Monte Carlo \cite{schweizer2012nonasymptotic,paulin2018error,mathews2022finite,lee2024convergence}. 

However, even simple multimodal distributions can violate the condition of balance between temperatures and cause sampling to be provably hard. In particular, \cite[Appendix F]{ge2018simulated} show an exponential query complexity lower bound for sampling from a $L^\iy$-perturbed mixture of two Gaussians with different covariances.
Here, the key difficulty in their setting is finding the components, whereas our results apply without issue.
Simulated tempering also fails for 
the mean-field Potts models \cite{bhatnagar2004torpid} (whereas our method will succeed, see \cref{thm:mf-potts}). 

The problem of sampling from multimodal distributions given \emph{warm starts} to the different modes has also been considered, for which the Annealed Leap Point Sampler \cite{tawn2021annealed,Roberts_Rosenthal_Tawn_2022} has been proposed. We note this is a weaker notion of ``advice" than ours, and guarantees have only been given in a limit where the components are approximately Gaussian.




\paragraph{Glauber dynamics in multimodal/metastable settings.} There has been a lot of work on understanding the behavior of Glauber dynamics with respect to metastability in spin systems, especially for complete graph models, random graphs, and the square lattice. See e.g.\ \cite{gheissari2022low,blanca2024mean,bovier2021metastability,cuff2012glauber,levin2010glauber,ding2009censored,ding2009mixing,galanis2024planting} for rigorous results. For the most part, this literature focuses on settings where there are a small number of metastable states which can be explicitly characterized, often taking advantage of symmetry considerations. For example, when the Ising model has no external field, it is symmetric under interchange of $+$ and $-$, so initializations like $(1/2)(\delta_{\vec 1} + \delta_{-\vec 1})$ are natural and have been studied in some of the aforementioned works. Generally speaking, our results hold in a general setting and do not rely on any structure of the underlying distribution besides the higher-order eigenvalue gap, but unlike those works we do not obtain explicit characterizations of the metastable states. In \cref{s:extra-examples}, we show how to combine our general technique with results from this literature in the case of the Curie-Weiss model, which lets us make more precise statements about the spectrum of the generator and mixing from non-data-based initialization. 

\paragraph{Learning from dynamics.} There have also been recent works on learning a distribution \emph{from} the Glauber dynamics, rather than from i.i.d.\ samples. For example, see \cite{bresler2017learning,gaitonde2024efficiently}. The work \cite{jayakumar2024discrete} is in spirit closely related---they study learning the Ising model when given i.i.d.\ samples from a metastable state (i.e. a region where Glauber dynamics becomes trapped for a long time). These results do not have new implications for the i.i.d.\ setting we study, but combining the ideas from our work with this setting may be an interesting direction for future work.  

\paragraph{Low-complexity Ising models.} The class of Ising models which are close to being low rank is significant, because informally these are the models for which the ``naive mean-field approximation'' from statistical physics is appropriate. (Naive mean field is, generally speaking, the most common type of variational approximation when doing variational inference in practice. See e.g.\ \cite{wainwright2008graphical} for some background.) One significance of this class of models is that it includes models in all of the high temperature, critical temperature, and low temperature regimes, and in particular many settings where the mixing time of natural Markov chains are exponential in the dimension due to multimodality or metastability. There have been many works in probability theory studying the structural properties of this class of models and generalizations---see e.g. \cite{basak2017universality,eldan2018gaussian,eldan2018decomposition,jain2019mean,austin2019structure,augeri2021transportation}. 

As in \cite{koehler2022sampling}, the class of models we consider for learning are somewhat broader, in that we allow the bulk of the spectrum to have diameter at most $1$ instead of requiring the bulk to be asymptotically negligible; this means the class also includes some models where naive mean field approximation is highly inaccurate (e.g. sparse models and spin glasses at high temperature). 

\paragraph{Theory for score matching.} There has been a lot of recent work on score matching, diffusion models, and related topics which we cannot exhaustively survey; instead, we mention a few relevant works. Denoising diffusion models are a popular approach to generative modeling which use approximations to the score function of the distribution convolved with different levels of Gaussian noise; recent works showed that these methods are robust to $L_2$-approximation of the score function (see e.g.\ \cite{chen2023sampling,chen2023improved,benton2023linear} and references within). The method we study is different in that only the original (``vanilla'') score function is needed; see \cite{koehler2023sampling} for more discussion. One of the motivations for score matching is that it can be easier to compute than the maximum likelihood estimator; for example, in some cases computing the MLE is NP-hard but vanilla score matching is statistically effective and computationally efficient \cite{pabbaraju2024provable,hyvarinen2005estimation,hyvarinen2007connections}. 
When the vanilla score function is estimated from data, it turns out that data-based initialization is not only needed for computational reasons---if Langevin dynamics on the estimated score function is run to stationarity, then in many multimodal settings the resulting distribution will be a poor estimate of the ground truth \cite{lee2022convergence,koehler2022statistical,balasubramanian2022towards}. If the distribution is unimodal, more specifically satisfies the log-Sobolev inequality, then the same work shows that vanilla score matching is statistically efficient even when Langevin is run to stationarity.

\paragraph{Concurrent work.} During the preparation of this manuscript, we were made aware of independent and concurrent work \cite{huang2024weak} 
which also gives guarantees for sampling mixture distributions using Langevin with data-based initialization, with a different proof technique based on weak Poincar\'e inequalities. We note that their number of samples in the data-based initialization (see \cite[Theorem 5.1]{huang2024weak}) has $\rc{p_*}$ dependence on the minimum mixture weight $p_*$, while ours depends only on the number of components $k$.

\subsection{Technical overview}
The proof of our main result, \cref{thm:main}, is conceptually simple. 
We aim to prove that the process initialized at the empirical samples contracts in $\chi^2$-divergence, which corresponds to the contraction in $L^2(\pi)$-norm where $\pi$ is the stationary distribution.
When the Markov process has a higher order spectral gap, i.e., $\la_{k+1}(-\sL)\ge \al$ where $\sL$ is the generator of the Markov process and $\la_{k+1}(-\sL)$ is the $(k+1)$-th smallest eigenvalue, this contraction holds for
any function $\phi$ which is orthogonal to the subspace $V$ spanned by the eigenfunctions corresponding to the $k$ smallest eigenvalues of $-\sL.$ 
The conclusion still holds if the projection of $\phi$ to $V$ has a sufficiently small norm, a condition we name \emph{eigenfunction balanced} (see \Cref{def:eigenfunction balanced}).

Hence, to prove rapid mixing from data-based initialization, we only need to show that the empirical distribution satisfies this condition. This is also the key technical challenge of our main result. 
For ease of presentation, we summarize our argument in the simpler finite-dimensional setting.
We observe that when $ \phi$ corresponds to the empirical distribution, the \emph{expected} projection\footnote{taken over the randomness of the empirical dataset} is precisely zero, due to the orthogonality of the eigenfunctions/eigenvectors. Our task is thus to establish a strong (Chernoff-like) concentration bound for this projection norm. A key technical difficulty is that we do not have high moment bounds; we only have a second moment bound due to orthonormality. A second moment bound typically leads to a weaker concentration bound, where the number of samples $n$ has linear dependency on the failure probability $\de,$ instead of the expected $\log(1/\delta)$ dependency. 

For concreteness, let $\pi$ be the stationary distribution of the Markov process and $(f_i)_{i=1}^{\iy}$ be an orthonormal basis of eigenvectors of $-\sL$ with eigenvalues $\la_1 = 0\leq  \la_2\leq \la_3\leq \cdots.$ Then for $y\sim\pi$
\[\E_{y\sim \pi}[\langle \delta_y , f_1 \rangle_{L^2(\pi)}]=\E_{y\sim \pi}[f_i(y)] =  \langle f_i, f_1 \rangle_{L^2(\pi)} = 0\]
since $f_1 \equiv 1.$ We can also bound the second moment of the projection by
\[\E_{y\sim \pi}[\langle \delta_y , f_1 \rangle_{L^2(\pi)}^2] =\E_{y\sim \pi}[f_i^2(y)] =  \langle f_i, f_i \rangle_{L^2(\pi)} = 1. \]
More generally, for $\phi =\rc n \sum_{i=1}^n \delta_{y_i}$ where $y_1, \cdots, y_n\sim  \pi$ are i.i.d.,
\[ \E_{y_1,\dots, y_n  \sim \pi\text{ i.i.d}}[\langle \phi,  f_1 \rangle_{L^2(\pi)}] = 0 \quad \text{and} \quad  \E_{y_1,\dots, y_n \sim \pi \text{ i.i.d}}[\langle \phi,  f_1 \rangle_{L^2(\pi)}^2] = 1/n.  \]
A naive application of standard concentration inequality such as Chebyshev's inequality gives us the following
\[\P_{y_1,\cdots,y_n \sim \pi \text{ i.i.d.}} [| \langle \phi,  f_1 \rangle_{L^2(\pi)}|\geq \etv ] \leq \de \]
when $n \geq \frac{1}{ \etv^2 \langle \phi,  f_1 \rangle_{L^2(\pi)}^2},$ which has sub-optimal dependency on $ \de.$

We obtain an exponential improvement on the $\delta$-dependency by restricting the domain to those with bounded projection norm. We then show a strong concentration bound on this restricted domain using standard concentration inequalities for bounded random variables  (Bernstein's inequality). This implies that the process initialized at samples from the restricted domain is rapidly mixing. By an appropriate choice of parameters, we can ensure that most of the samples will be from the restricted domain, and thus the result followed from a standard comparison argument.

An important case satisfying the higher-order spectral gap is when the stationary distribution $ \pi$ admits a decomposition into a \emph{mixture} of distributions each satisfying a Poincar\'e inequality (see \Cref{l:egap}). For this case, our mixing time bound matches the state-of-the-art worst-case-start mixing for a single distribution satisfying the Poincar\'e inequality (see \Cref{l:ld}(1)).
If we further assume that each component of the mixture satisfies the stronger log-Sobolev inequality, then using the hyper-contractivity argument from \cite{lee2024convergence}, we can obtain a tighter bound on the mixing time. For the continuous Langevin diffusion on mixtures of distributions satisfying log-Sobolev inequality, our mixing time bound has an optimal dependency on all parameters (see \Cref{t:ld-better}).

\paragraph{Markov chain perturbation.} In many applications, we suffer errors when implementing the Markov process. For example, we can only implement a discretization of the continuous Langevin diffusion, i.e., the Langevin Monte-Carlo. Furthermore, in learning problems, we might not have access to the chain transition probabilities but need to estimate them from data, e.g., using score matching or pseudolikelihood estimation. We show that our analysis is robust to such perturbations. 

In particular, we establish that data-based initialization is a natural and elegant way to exploit the guarantees on transition probabilities estimators learned from data, which is typically of the form $ \E_{x\sim \pi}[\text{dist}(\hat{u}(x), u(x))]\leq \esc$ where $\hat{u}, u$ stand for the estimated and the true transition probabilities, respectively. A priori, it is unclear how to use such a guarantee in the initial stage of the Markov process when the distribution is very far from the stationary distribution $\pi.$ Even when $\pi$ has rapid convergence to stationarity from \emph{worst-case} start, most previous works \cite{lee2022convergence}
can only handle such perturbation by assuming that the process is initialized at a a distribution $\nu$ s.t. $\chi^2(\nu||\mu)$ is small.  
However, 
the empirical distribution \emph{does not} satisfy this condition. For example, if $\pi$ a continuous distribution and $\nu$ is the empirical distribution then $ \chi^2(\nu||\pi) = +\infty$ since $\nu$ is discrete. If $\pi$ 
is the uniform distribution over the hypercube $\set{\pm 1}^d $ and $ \nu$ is the empirical distribution formed by $n= o(2^d)$ samples, then $ \chi^2(\nu ||\pi) =\Omega(2^d).$ Nevertheless, the empirical distribution is closely related to $\pi$ in an average sense, and \cite{koehler2023sampling} exploited this fact in their perturbation analysis; however, their analysis is lossy, incurring an extra $ \poly(1/\etv)$ factor in the number of samples. 

In \Cref{thm:approximate-sample-init}, we directly reduce the perturbation analysis from processes initialized at data samples to the perturbation analysis when initialized at \emph{the stationary distribution}. Unlike \cite{koehler2023sampling}, we do not incur any loss in the number of samples (see \Cref{thm:langevin discretization}). The key idea is to bound the expected TV distance between two processes $X_t$ and $\tilde{X}_t$ when initialized at a data-sample $Y\sim \pi$ by the KL-divergence when initialized at the stationary distribution $\pi$ by applying Jensen's inequality, Pinsker's inequality and the chain rule for KL-divergence.
\begin{align*}
\E_{Y\sim \pi} [
\TV
    (\mathcal L((X^{Y}_t)_{0\le t\le T}), \mathcal L((\tilde X^{Y}_t)_{0\le t\le T}))]
\le \sqrt{\rc 2 \KL(\mathcal L((X^{\pi}_t)_{0\le t\le T})\| \mathcal L((\tilde X^{\pi}_t)_{0\le t\le T}))}
\end{align*}
We note that we can bound the second moment using a similar argument. We can then obtain a strong concentration bound using Bernstein's inequality and the fact that $\TV$-distance is bounded by $1.$ 

Finally, we perform perturbation analysis for processes initialized at the \emph{stationary distribution}. For continuous processes such as the Langevin diffusion with smooth stationary distribution, such a bound follows from the Girsanov's theorem (see \Cref{l:ld-score-error}). In our application, the stationary distribution is not necessarily smooth, but is a mixture of smooth distributions satisfying a Poincar\'e inequality. In that case, we can establish quantitatively similar results using higher moment bounds implied by the Poincar\'e inequality on each component (see \Cref{l:ld-score-error-mix}). For the Glauber dynamics, we derive a qualitatively similar result in \Cref{lem:gibbs-approximate} when the transition probabilities, i.e., the conditional marginals, are estimated using pseudo-likelihood. 



\paragraph{Application to low-complexity Ising models.} The learning result follows from our general theory provided we can: (1) prove such a distribution has a higher-order spectral gap, and (2) estimate the needed Glauber transitions from data. For (1), we do this using a two-step argument; first, we use techniques from \cite{koehler2022sampling} and a result from \cite{anari2024trickle} to prove that such an Ising model can be approximately decomposed into a small mixture of rapidly mixing Ising models. Because our approximate decomposition has a density which is within a constant factor of the true model, we can boost this to an exact mixture decomposition into a mixture of Poincar\'e distributions by using the robustness of spectral gap to small changes of measure. For problem (2) of learning the transitions from data, this is exactly the problem solved by pseudolikelihood estimation, which we can analyze using standard symmetrization techniques from statistical learning theory.

\subsection{Organization}
We cover the mathematical preliminaries in \cref{s:prelim}. \cref{s:main} is devoted to the proof of our main result, \cref{thm:main}. In \cref{s:perturb}, we show the robustness of Langevin and Glauber dynamics in our setting to perturbations in the dynamics. \cref{s:langevin} covers the application of our theory to Langevin dynamics on mixtures, and \cref{s:learning-ising} illustrates the application to Glauber dynamics for sampling and learning Ising models. Finally, in \cref{s:extra-examples} we give some examples of non-sample initializations which also satisfy eigenfunction balance.
\section{Preliminaries}\label{s:prelim}
For a vector-valued function $f,$ let $f_{r+1:s}(y)$ define the vector $(f_{r+1}(y),\ldots, f_s(y))\in \R^{s-r}$.
\subsection{R\'enyi divergence}
The R\'enyi divergence, which generalizes the more well-known KL divergence, is a useful technical tool in the analysis of the Langevin diffusion, see e.g., \cite{vempala2019rapid}.
The R\'enyi divergence of order $q \in (1, \infty)$ of $\mu$ from $\pi$ is defined to be
\begin{align*}
  \Renyi_q(\mu || \pi) &= \frac{1}{q-1} \ln \E_{\pi} \left[\left( \frac{d\mu}{d\pi}\right)^q\right] 
  =   \frac{1}{q-1} \ln \E_{\mu} \left[\left( \frac{d\mu}{d\pi}\right)^{q-1}\right]  . 
\end{align*}
The limit $\Renyi_q$ as $q \to 1^+$ is the Kullback-Leibler divergence $\D_{\KL}(\mu || \pi)  = \int \log \pf{d\mu}{d\pi} d\mu,$ thus we write $ \Renyi_1 (\cdot ) = \D_{\KL}(\cdot).$ R\'enyi divergence increases as $q$ increases, i.e. $\Renyi_q \leq \Renyi_{q'} $ for $1\leq q \leq q'.$ We note some basic facts about R\'enyi divergence.
\begin{lemma}[Weak triangle inequality, {\cite[Proposition 11]{mironov2017renyi}}] \label{lem:weak triangle inequality}
For $q > 1$ and any measure $\nu$ absolutely continuous with respect to measure $\mu$,
for any $a,b>1$ such that $\rc a + \rc b=1$, 
\[
\Renyi_q(\nu || \mu) \leq \frac{aq - 1}{a(q - 1)} \Renyi_{aq}(\nu || \nu') + \Renyi_{b(q-1)+1} (\nu' || \mu).
\]
\end{lemma}
\begin{lemma}[Weak convexity of R\'enyi entropy] \label{lem:weak convexity}
 For $q> 1$, if $\mu$ is a convex combination of $\mu_i$, i.e. $\mu = \sum_{i=1}^k p_i \mu_i$ for $p_i>0$ such that $\sum_{i=1}^k p_i=1$, and $\nu\ll \mu$, then
 \[ \E_{\nu}\left[\left(\frac{d\nu }{d \mu}\right)^{q-1}\right] \leq \sum_{i=1}^k  p_i \E_{\nu} \left[\left(\frac{d\nu }{d \mu_i}\right)^{q-1}\right]. \]
 Consequently,
 $\Renyi_q (\nu || \mu) \leq \max_i\Renyi_q (\nu || \mu_i)  $. If instead $\mu\ll \nu$, we have $ \Renyi_q(\mu || \nu) \leq \max_i\Renyi_q (\mu_i || \nu) $.
 \end{lemma}
 \begin{proof}
 By H\"older's inequality, we have $\nu$-almost everywhere that 
 \[\pa{\sum_{i=1}^k p_i \dd{\mu_i}{\nu}}^{q-1} \left( \sum_{i=1}^k p_i \pa{\dd{\nu}{\mu_i}}^{q-1} \right) \geq \pa{\sum_{i=1}^k p_i}^q = 1.\]
 Thus, we have $\nu$-almost everywhere that 
 \[ \left(\frac{d\nu}{d\mu}\right)^{q-1}\leq \sum_{i=1}^k p_i \left(\frac{d\nu}{d\mu_i}\right)^{q-1}. \]
 Taking expectation in $\nu$ gives the first statement. 
 For the second statement, by bounding the convex combinations above with the maximal term,
 \[ \Renyi_q(\nu || \mu) = \frac{\ln\E_{\nu}[(\frac{d\nu}{d \mu})^{q-1}] }{q-1} \leq \frac{\ln (\max_{i} \E_{\nu}[(\frac{d\nu }{d \mu_i})^{q-1}] ) }{q-1} = \max_i \Renyi_q (\nu ||\mu_i). \]
For the final statement, by Jensen's inequality and convexity of $x^q$,
\begin{align*}
    \ln \E_\nu \ba{\pa{\dd \mu\nu}^q} 
    &\le \ln \sum_{i=1}^k p_i \E_\nu \ba{\pa{\dd{\mu_i}{\nu}}^q} \le \ln \max_i \E_\nu \ba{\pa{\dd{\mu_i}{\nu}}^q} = \max_i \Renyi_q(\mu_i\|\nu). 
\end{align*}
 \end{proof}
 
\subsection{Functional inequalities}
 For nonnegative $ f: \R^d \to \R_{\geq 0},$ define the entropy of $f$ with respect to probability distribution $\pi$ to be \[ \Ent_{\pi} [f]  = \E_{\pi} [f\ln (f/ \E_{\pi}[f]) ]. \]

\paragraph{Functional inequalities for the Langevin diffusion.}
We say $\pi$ satisfies a log-Sobolev inequality (LSI) with constant $\CLS$ with respect to the Langevin semigroup if for all smooth functions $f$,
\[\Ent_{\pi} [f^2] \leq 2 \CLS \E_{\pi}[ \norm{\nabla f}^2 ] \]
and $\pi$ satisfies a Poincar\'e inequality (PI) with constant $\CP$ if
\[\Var_{\pi}[f] \leq 2 \CP \E_{\pi}[ \norm{\nabla f}^2 ]. \]
The log-Sobolev inequality implies the Poincar\'e inequality: $\CP \leq \CLS.$
Due to the Bakry-\'Emery criterion \cite{bakry2006diffusions}, if $\pi$ is $\alpha$-strongly log-concave then $ \pi $ satisfies LSI with constant $\CLS=1/\alpha.$ 

LSI and PI are equivalent to statements about exponential ergodicity of the continuous-time Langevin diffusion, which is defined by the Stochastic Differential Equation (SDE)
\[ d\bar X_t^{\pi} = \nabla \log \pi(\bar X_t^{\pi})\, dt + \sqrt{2}\, dB_t. \]   
Specifically, let $\pi_t$ denote the law of the diffusion at time $t$ initialized from $\pi_0$. Then a LSI is equivalent to the inequality
\[ \KL (\pi_t\| \pi) \leq \exp(-2t/\CLS) \KL(\pi_0\| \pi) \]
holding for an arbitrary initial distribution $\pi_0$.
Similarly, a PI is equivalent to
$\chi^2 (\pi_t\| \pi) \leq \exp(-2t/\CP) \chi^2(\pi_0\| \pi)$. Here $\KL(P\| Q) = \E_P[\ln \frac{dP}{dQ}]$ is the Kullback-Liebler divergence and $\chi^2(P\| Q) = \E_Q[(\dd PQ - 1)^2]$ is the $\chi^2$-divergence.

\paragraph{Markov semigroups and generators.} Our results apply to general diffusions besides Langevin. In particular, they also have interesting and useful applications for the Glauber dynamics or Gibbs sampler, which resamples the coordinates of the random vector one at a time. In order to state results in the appropriate level of generality, we recall the definition of a Markov semigroup.  See \cite{bakry2014analysis,van2014probability} for more background.

A stochastic process $(X_t)_{t \ge 0}$ is a \emph{time-homogenous Markov process} if for every $t \ge 0$, there exists a linear operator $P_t$ such that for all $s \ge 0$ and bounded measurable functions $f$,
\[ \mathbb E[f(X_{s+t}) \mid (X_r)_{r \le s}] = (P_t f)(X_s). \]
We define the corresponding operator on measures $\mu P_t$ or $P_t^*\mu$ as the distribution of $X_{s+t}$ when $X_s\sim \mu$; this leads to the identity $\int \mu (P_tf) = \int (\mu P_t) f$ where $f$ is bounded measurable and $\mu$ is a measure.

Suppose that the stochastic process has stationary measure $\pi$. 
Given such a process, we can define its infinitesimal generator 
as $\sL$ such that 
$P_t = e^{t \sL}$.
We say $P_t$ is \emph{reversible} if $\sL$ is a self-adjoint operator on $\L_2(\pi)$ (this is discussed more later). 
The \emph{Dirichlet form} is then defined as
\[ \mathcal{E}(f,g) = -\langle f, \sL g \rangle_{\L_2(\pi)}. \]
The \emph{spectral gap} of the generator is the maximal $\gamma \ge 0$ such that for all test functions $f$,
\[ \gamma \Var(f) \le \mathcal{E}(f,f). \]
To be consistent with our notation above, we call $1/(2\gamma)$ the Poincar\'e constant; the above inequality is the general form of the Poincar\'e inequality for Markov semigroups. As before, the Poincar\'e constant characterizes ergodicity in $\chi^2$-divergence. 

For example, for the Langevin diffusion described above, the generator is 
\begin{equation}\label{eqn:langevin-generator}
\sL f = \langle \nabla \log \pi, \nabla f \rangle_\pi + \Delta f 
\end{equation}
where $\Delta f = \sum_i \partial_i^2 f$ is the usual Laplacian, and we can compute that $\mathcal E(f,f) = \mathbb E \|\nabla f\|^2$. Similarly, the semigroup for the 
Glauber dynamics with stationary measure $\pi$ over a measure space $\bigotimes_{i = 1}^n \Omega_i$ is generated by
\[ \sL = \sum_{i = 1}^n (E_i - I) \]
where $E_i f = \E_{X \sim \pi}[f(X) \mid X_{\sim i}]$ for all $1 \le i \le n$, and $X_{\sim i}$ denotes all the coordinates of $X$ except $i$. 
Note that one unit of continuous time for the Glauber dynamics corresponds to $\textup{Poisson}(n)$ discrete updates.
The generator of the Glauber dynamics is a self-adjoint operator on $L_2(\pi)$ and the generator of the Langevin dynamics is essentially self-adjoint;  we define these terms more precisely below.

\subsection{Spectral theory}


Because operators like the generator of the Langevin diffusion are (essentially) self-adjoint operators on infinite-dimensional spaces, we will need to use some terminology and tools from functional analysis to analyze them in full generality. Readers who are only interested in operators with discrete spectrum (e.g., those on finite-dimensional spaces, or for Langevin dynamics on distributions with sufficiently rapid tail decay---see Corollary 4.10.9 of \cite{bakry2014analysis}) can largely ignore this section.

In the general case the spectrum of any self-adjoint operator  can be decomposed into two disjoint parts: a discrete spectrum (the only part present in the finite-dimensional case) and an essential spectrum, which are both defined precisely below. For example, the spectrum of the Laplacian in Euclidean space is $[0,\infty)$ so it is all essential spectrum. For a simple example in our context, the spectrum of the generator of the Langevin diffusion on the symmetric exponential density $\mu(x) = \frac{1}{2} e^{-|x|}$ is not discrete. (This is also true if the density is modified in a neighborhood of zero to be smooth. See Section 4.4.1 of \cite{bakry2014analysis} for discussion.) Nevertheless, this density admits a spectral gap in the appropriate sense, which we also define below. 

First, we briefly review the most relevant facts from functional analysis, see e.g. \cite{reed1980methods,bakry2014analysis,hall2013quantum,teschl2014mathematical} for details. For an operator $A$ defined on $\operatorname{Dom}(A)$, we let $A^*$ denote its adjoint. We say that operator $A$ is \emph{self-adjoint} if $A = A^*$ and $\operatorname{Dom}(A) = \operatorname{Dom}(A^*)$. An operator is said to be \emph{essentially self-adjoint} if its closure is self-adjoint.
\begin{theorem}[Special case of Corollary 3.2.2 of \cite{bakry2014analysis}]
Suppose that $\mu$ is a probability measure on $\mathbb{R}^n$ with a smooth density. Then the generator $\sL$ of the corresponding Langevin diffusion, defined on the set of smooth compactly supported functions, is essentially self-adjoint.
\end{theorem}
\begin{theorem}[Spectral theorem, Theorem VIII.6 of \cite{reed1980methods}, see also Chapter A.4 of \cite{bakry2014analysis}]
There is a one-to-one correspondence between self-adjoint operators $A$ and projection-valued measures $\pi$ on a Hilbert space $H$, given by
\[ A = \int \lambda\, d\pi_{\lambda}. \]
\end{theorem}
With the notation of the above theorem, we have the functional calculus $g(A) = \int g(\lambda)\, d\pi_{\lambda}$ and in particular
\[ P_t f = e^{t \sL} f = \int e^{t \lambda}\, d\pi_{\lambda}; \]
see Chapter VIII.3 of \cite{reed1980methods} or Chapter A.4 of \cite{bakry2014analysis}. This enables us to generalize results from the case of a completely discrete spectrum in a natural way.

\paragraph{Discrete and essential spectrum.} The support of the projection-valued measure $\pi$ is called the \emph{spectrum} of the operator $A$ and denoted $\sigma(A)$. Furthermore, the spectrum can be decomposed as a disjoint union
\[ \sigma(A) = \sigma_d(A) \cup \sigma_{\textup{ess}}(A) \]
where $\sigma_d(A)$ is called the \emph{discrete spectrum}, and consists of isolated points in the spectrum which are also required to be eigenvalues of $A$ with finite multiplicity, and $\sigma_{\textup{ess}}(A) = \sigma(A) \setminus \sigma_d(A)$ is called the \emph{essential spectrum}. Note that the discrete spectrum may be a proper subset of the point spectrum of $A$, i.e., of the full set of eigenvalues of $A$.

\paragraph{Example: finite-dimensional case.} For example, in the finite-dimensional case where $A$ is an operator on $\mathbb R^n$ the above formulation of the spectral theorem reduces to the fact that
\[ A = \sum_{i = 1}^n \lambda_i v_i v_i^T \]
where $\lambda_i$ is the $i$th eigenvalue, $v_i$ is the $i$th eigenvector, and $v_i v_i^T$ is the (rank one) projection operator onto the span of $v_i$. The spectrum will just be the set of eigenvalues of $A$, which is necessarily discrete.

\paragraph{Higher-order spectral gap.} For finite dimensional operators, the natural higher-order spectral gap assumption is that $\lambda_{k + 1}(A) \ge \al$ for some $k \ge 1$ and $\al > 0$.
We can now write down the general analogue for infinite-dimensional operators.
\begin{definition}\label{def:higher-order-gap}
We say a positive-semidefinite self-adjoint operator $A$ has a spectral gap after eigenvalue $k$ (or a $k$th order spectral gap) of size at least $\al$ for $k \ge 1$ and $\al > 0$ if 
\[ \sigma_{\textup{ess}}(A) \subset [\al, \infty) \]
and the projection $\pi([0,\al))$ has rank at most $k$, where $\pi$ is the spectral measure. (In other words, the latter condition says there are at most $k$ eigenvalues, counted with multiplicity, in $[0,\al)$.) We will also write this condition as $\la_{k+1}(A)\ge \al$. 
\end{definition}
The higher-order spectral gap defined in the above sense has the same min-max variational characterization as in the finite dimensional case: see Theorem 4.10 of \cite{teschl2014mathematical}. For example, this means that the characterization of the spectral gap in terms of the Poincar\'e functional inequality holds in general. 

Similar to \cite{teschl2014mathematical}, we will generally use the following notational convention for infinite-dimensional operators where the spectrum is bounded below: we write that the spectrum of a self-adjoint operator $A$ below the essential spectrum is $\lambda_1 \le \lambda_2 \le \cdots$, where this is a list first consisting of the ordered elements of the discrete spectrum below $\inf \sigma_{\textup{ess}}(A)$, and then if the first list is finite, this is followed by $\inf \sigma_{\textup{ess}}(A)$ repeated infinitely many times. This convention generally enables us to write the same statement for finite and infinite-dimensional settings. 

\section{Markov chains with data-based initialization: \mbox{Mixing} given warm start}\label{s:main}
The following general theorem analyzes the performance of a generic Markov chain with a higher-order spectral gap, showing it will mix well when started from a small number of samples from the true distribution.
A key assumption for the theorem is the guarantee that dynamics run for some initial time period $t_0$, started from a typical sample, reaches a distribution with bounded $\chi^2$-divergence to the stationary measure $\pi$. Proving this assumption requires a different analysis for different Markov chains.
We will later show how to apply this result in natural settings like mixtures of log-concave measures and in statistical physics models. 
\begin{theorem} \label{thm:main}
Let $ \sL$ be the self-adjoint generator of a reversible Markov semigroup $P_t = e^{t\sL}$ with stationary distribution $\pi$ defined over $\mathcal D$. 
Suppose that $- \sL$ satisfies the $k$th order spectral gap condition from \cref{def:higher-order-gap}, $\la_{k+1}(-\sL)\ge \al>0$, where $k\ge 1$.

For $y\in \mathcal{D}$ and $t \ge 0$, let $\rho^y_t = \delta_y P_t$ be the marginal law of the Markov chain generated by $\sL$ initialized at $\delta_y.$ 
Let $y_1, \dots, y_n$ be i.i.d.\ samples from $ \pi$ and let $U_{\text{sample}}$ be the multiset of $y_1, \ldots, y_n$. 
Consider the Markov process generated by $\sL$ initialized at $\mu_0 = \frac{1}{n} \sum_{j=1}^n \delta_{y_j}$, and let $\mu_t = \mu_0 P_t = \frac{1}{n} \sum_{j = 1}^n \rho_t^{y_j}$ be the marginal law of the process at time $t$ conditional on $U_{\text{sample}}$.

Suppose there exists time $t_0 \ge 0,$ parameter $R \ge 0$ and a set $\Ombd$ such that
\[\forall y \in \Ombd: \chi^2(\rho^y_{t_0} ||\pi)\leq R, \]
and $ \pi(\Ombd^c) \leq \frac{\etv^2}{16k}.$

Then for $t\ge t_0$, with probability $\ge 1-k\exp\pa{-\Om\pf{n\etv^2}k}$, 
\[
\TV(\mu_t,\pi)\le  \sqrt{\fc{\etv^2}4 + e^{-\al (t-t_0)}R} + \fc{\etv}{16k}.
\]
In particular, for $n= \Om\pa{\fc{k}{\etv^2}\ln \pf k\de}$  and 
\[ t\geq T:= t_0 + \fc1\al \ln \pf{2R}{\etv^2}, \] with probability at least $1- \de$ over the randomness of $U_{\textup{sample}},$
\[ \TV (\mu_t, \pi) \leq \etv. \]
\end{theorem}

    
\subsection{Eigenvalue gap for mixtures}
We now show that the assumption in \Cref{thm:main} is satisfied in the case of mixtures of distributions satisfying a Poincar\'e inequality. More precisely, this holds for arbitrary semigroups as long as the Dirichlet form of the mixture dominates the average Dirichlet form of the components (\cref{e:mix-assm2} below); this holds automatically for common semigroups like the Langevin and Glauber dynamics.
\begin{lemma}[{Eigenvalue gap for Markov chain on mixtures}]
\label{l:egap}
    Let $\sL$, and $\sL_i, 1\le i\le k$ be generators of Markov processes on a measurable space $(\cal D, \mathscr F)$ with stationary distributions $\pi$ and $\pi_i, 1\le i\le k$, respectively. Suppose that there are weights $w_i>0$ such that 
    \begin{align}
    \label{e:mix-assm1}
        \pi &= \sumo ik w_i\pi_i\\
    \label{e:mix-assm2}
        \forall f\in \operatorname{Dom}(\sL) \subseteq \L_2(\pi), \quad 
        \an{f, -\sL f}_\pi & \ge \sumo ik w_i \an{f,-\sL_i f}_{\pi_i}.
    \end{align}
    Suppose that $\la_2(-\sL_i)\ge \al$ for each $1\le i\le k$ (i.e., each $\pi_i$ satisfies a Poincar\'e inequality with constant $\rc \al$). Then $\la_{k+1}(-\sL)\ge \al$.

    In particular, \eqref{e:mix-assm2} holds in the following cases.
    \begin{enumerate}
        \item $\cal D= \R^k$ and $\sL$, $\sL_i, 1\le i\le k$ are the generators of Langevin diffusion with stationary distributions $\pi$, $\pi_i, 1\le i\le k$.
        \item $\cal D= \bigot_{i=1}^n \Om_i$ and  $\sL$, $\sL_i, 1\le i\le k$ are the generators of Glauber dynamics with stationary distributions $\pi$, $\pi_i, 1\le i\le k$.
    \end{enumerate}
\end{lemma}
If the generator does not have a discrete spectrum, the assumption and conclusion should be interpreted in the sense of \cref{def:higher-order-gap}; see the discussion after the definition about our notational convention.
\begin{proof}
    This is a generalization of the proof for Langevin dynamics in \cite[Lemma 6.1]{ge2018beyond}\footnote{See the arxiv version \arxiv{1710.02736}}. 
    Let $V=\spn\bc{\dd{\pi_i}{\pi}:1\le i\le k}$, and take $f\in\operatorname{Dom}(\sL) \subseteq \L_2(\pi)$ such that $f\in V^\perp$ (with respect to $\an{\cdot,\cdot}_\pi$).  
    This means that 
    \[
\E_{\pi_i}f = \E_{\pi} \dd{\pi_i}{\pi}f = 0 \text{ and }\E_\pi f= 0.
    \]
    Then by the law of total variance and the spectral gap assumption for each $\pi_i$, 
    \begin{align*}
    \ve{f}_\pi^2 = 
        \Var_\pi(f) &= \sumo ik w_i (\E_{\pi_i} f - \E_\pi f)^2 + \sumo ik w_i \Var_\pi(f)\\
        &\le 0 + \sumo ik w_i \rc \al \an{f, -\sL_i f}_{\pi_i} \\
        &\le \rc \al \an{f, -\sL f}_{\pi}. 
    \end{align*}
    Since we proved this for all $f$ orthogonal to a subspace of dimension $k$, the desired conclusion follows from the variational characterization of eigenvalues,
    \[
\la_{k+1}(-\sL) = \maxr{\text{subspace }S\subeq \L^2(\pi)}{\dim S = k} \min_{f\perp_\pi S} \fc{\an{f, -\sL f}_\pi }{\ve{f}_\pi^2}\ge \al.
    \]
    Finally, we note that for Langevin diffusion, \eqref{e:mix-assm2} holds with equality:
    \[
\int_{\R^n} \ve{\nb f}^2 \,d\pi = 
\sumo ik \int_{\R^n} w_i  \ve{\nb f}^2 \,d\pi_i.
    \]
    For Glauber dynamics, \eqref{e:mix-assm2} holds as well---see \cite{lee2024convergence} or Lemma 29 of \cite{anari2024trickle}. In short, it follows from the equation $\mathcal{E}(f,f) = \sum_{i = 1}^n \E \Var(f(X) \mid X_{\sim i})$ for the Dirichlet form of the Glauber dynamics, combined with the law of total variance over the choice of mixture component. 
\end{proof}
\begin{remark}
The comparison of Dirichlet forms, \cref{e:mix-assm2}, also holds for a Metropolis-Hastings chain with fixed proposal \cite{lee2024convergence}.
\end{remark}

\subsection{Mixing from a balanced initialization}
To prove \Cref{thm:main}, we use the following \Cref{l:small-sum} which says that as long as a set of points $y_1,\ldots,y_n$ satisfy a natural balance condition \eqref{eqn:balance} defined in terms of the eigenfunctions of the generator, a higher-order spectral gap leads to rapid contraction of the $\chi^2$-divergence along the diffusion. 
Then 
in \Cref{lem:vec bound},
we show the natural balance condition is satisfied with high probability using matrix Bernstein. 
\begin{definition} \label{def:eigenfunction balanced}
Let $\sL$ be the self-adjoint generator of a reversible Markov semigroup on $\cal D$, and 
        let $f_1\equiv 1,f_2,\ldots, f_k$ be the eigenfunctions corresponding to (discrete) eigenvalues $0=\la_1\le \la_2\le \cdots\le \la_k$ of $-\sL$. 
        A distribution $\mu_0$ on $\cal D$ is \vocab{$(k,\ep)$-eigenfunction balanced} if
        \[
    \ve{\E_{Y\sim \mu_0} [f_{2:k}(Y)]}\le \ep.
        \]
\end{definition}
\begin{lemma}[Mixing under eigenfunction balance]\label{l:small-sum}
Let $\sL$ denote the generator of a Markov semigroup $P_t = e^{t\sL}$ with stationary distribution $\pi$ defined over $\mathcal D$, and suppose that 
$- \sL$ satisfies the $k$th order spectral gap condition from \cref{def:higher-order-gap}, $\la_{k+1}(-\sL)\ge \al>0$.

Consider an initial distribution $\mu_0$ and define $\mu_t = \mu_0P_t$.
    Let $\ep > 0$ and $t_0 \ge 0$ be arbitrary such that $ \chi^2(\mu_{t_0}|| \pi) < \infty$ and 
    \begin{equation}\label{eqn:balance general}
\ve{\E_{Y\sim \mu_0}\left[ f_{2:k} (Y)\right]} \le \ep.
    \end{equation}
 Then for $t\ge t_0$,
\[
\chi^2(\mu_t\|\pi) \le \ep^2 + e^{-\al(t-t_0)}\chi^2(\mu_{t_0} \| \pi). 
\]
 In particular, for $t \ge t_0 +\rc \al \ln \pf{\chi^2(\mu_{t_0}|| \pi)}{\ep^2},$ we have that
    $
\chi^2 (\mu_t \|\pi)\le 2 \ep^2
    $.
\end{lemma}
In particular, if $\mu_0= \rc n \sumo jn \de_{y_j}$ for some $y_1,\ldots, y_n \in \mathcal D$ then
\Cref{eqn:balance general} becomes
\begin{equation}\label{eqn:balance}
\ve{\rc n \sumo jn f_{2:k} (y_j)} \le \ep.
    \end{equation}

\begin{proof}
First we prove the result in the special case of a discrete spectrum.
Let $(f_i)_{i=1}^{\iy}$ be the eigenfunctions of $\sL$ with eigenvalues $-\la_1 = 0> - \la_2\ge -\la_3\ge\cdots$. 
Then for $t \ge t_0$, we have
\[
\dd{\de_{y}P_t}{\pi}(x) = \sumo i{\iy}e^{-\la_i t}f_i(y)f_i(x).
\]
Informally, this is because we would have $\dd{\de_{y}}{\pi}(x) = \sumo i{\iy} f_i(y)f_i(x)$, if we use the  ``calculation'' $\an{\dd{\de_y}{\pi}, f}_\pi = f(y)$. 
More formally, this is because by functional calculus, $P_t = e^{-t \sL} = \sum_{i = 1}^{\infty} e^{-\lambda_i t} f_i f_i^*$ where $f_i^* g = \langle f, g \rangle_{L_2(\pi)}$, i.e. $f_i^*$ is the $L_2(\pi)$-adjoint of $f_i$. Hence $\langle \frac{d \delta_y P_t}{d\pi}, g \rangle_{L_2(\pi)} = \de_y P_t g =  \sum_{i = 1}^{\infty} e^{-\lambda_i t} f_i(y) \langle f_i, g \rangle_{L_2(\pi)}$ for all $g \in L_2(\pi)$, so as claimed $\frac{d \delta_y P_t}{d\pi}(x) = \sum_{i = 1}^{\infty} e^{-\lambda_i t} f_i(y) f_i(x)$.


By linearity, 
we have 
\[ \dd{\mu_0P_t }{\pi} (x) =\sumo i{\iy}e^{-\la_i t} \E_{y\sim \mu_0}[f_i(y)]f_i(x)  = \sumo i{\iy}e^{-\la_i t}C_i f_i(x) = 1+\sum_{i=2}^{\infty}e^{-\la_i t}C_i f_i(x) \]
where $C_i= \E_{y\sim \mu_0}[f_i(y)] $ and the final equality follows from $ \lambda_1 = 0$ and $f_1 \equiv 1.$

We know that $ \langle f_i, f_j \rangle_{L^2(\pi)} = \begin{cases} 1 &\text{ if } i = j \\ 0 &\text{ else} \end{cases}$ thus
\begin{align*}
\chi^2(\mu_t \| \pi) = 
\ve{\dd{\mu_t}{\pi} -1}_{L^2(\pi)}^2 = 
\ve{g_t}_{L^2(\pi)}^2 + \ve{h_t}_{L^2(\pi)}^2.
\end{align*}
where $g_t = \sum_{i=2}^k e^{-\la_i t}C_i f_i(x)  $ and $ h_t = \sum_{i=k+1}^{\infty} e^{-\la_i t}C_i f_i(x) .$ Next, we bound 
\begin{equation}\label{eq:bounded gt}
   \begin{split}
\ve{g_t}_{L^2(\pi)}^2 &= \ve{\sum_{i=2}^k e^{-\la_i t} C_i  f_i}_{L^2(\pi)}^2
=  \sum_{i=2}^k e^{-2\la_i t} \cdot \norm{C_i f_i}_{L^2(\pi)} = \sum_{i=2}^k e^{-2\la_i t} \cdot \ab{C_i}^2\\
&\le  \sum_{i=2}^k  \ab{C_i}^2 = \ve{\E_{y\sim \mu_0} f_{2:k}(y)}^2 
\leq \ep^2,
   \end{split}
\end{equation}
where the first inequality follows from $ \lambda_i \geq 0$, the second equality from orthogonality of the eigenvectors $f_i$, the third equality from $ \norm{f_i}_{L^2(\pi)}=1,$ and the fourth equality from $\E_{y\sim \mu_0} f_{2:k}(y)  = [ C_2 \dots C_k]^\intercal.$ 
 Further,
\begin{align*}\ve{h_t}_{L^2(\pi)}^2= \sum_{i=k+1}^{\infty} e^{-2\la_i t} \cdot \ab{C_i}^2 &\leq e^{-\al (t-t_0)} \cdot \sum_{i=k+1}^{\infty} e^{-2\la_i t_0} \cdot \ab{C_i}^2 \\
&= e^{-\al(t-t_0)} \ve{h_{t_0}}_{L^2(\pi)}^2 \leq e^{-\al (t-t_0)} \chi^2(\mu_{t_0}||\pi) \end{align*}
where the inequality follows from $\lambda_i \geq \al$ for $i\geq k+1.$ Thus
\begin{align*}
\chi^2(\mu_t \| \pi) \le 
\ep^2 +  e^{-\al (t-t_0)} \chi^2(\mu_{t_0}||\pi) \le 
2\ep^2
\end{align*}
for $ t\ge t_0 + \rc \al \ln \pf{ \chi^2(\mu_{t_0}|| \pi)}{\ep^2}.$

In the case of a general operator satisfying \cref{def:higher-order-gap}, the argument works the same way except that we replace the tail sum $\sum_{i = k + 1}^{\infty} e^{-\lambda_i t} f_i f_i^*$ in the representation of $e^{t \mathcal L}$ by the integral $\int e^{-t \lambda} d\pi_{\lambda}$ over the projection-valued measure $\pi$ corresponding to $\mathcal L$. By using the functional calculus, we can still prove in the same way that $\|h_t\|^2_{L^2} \le e^{-(t - t_0)/C} \|h_{t_0}\|_{L^2}^2$, which controls the contribution to the $\chi^2$-divergence from the non-leading eigenvalues.
\end{proof}

\subsection{Ensuring balance}
We will ensure the balance condition \eqref{eqn:balance} when initialized at the empirical distribution using the following.
\begin{theorem}[{Matrix Bernstein, \cite[Theorem 6.1.1]{tropp2015introduction}}]
\label{t:mat-bern}
    Let $S_k$, $1\le k\le n$ be independent random matrices with dimension $d_1\times d_2$, with $\E S_k = 0$ and $\ve{S_k}\le L$ for each $k$. Let $Z=\sumo kn S_k$, and define
    \[
v(Z)= \max\bc{\ve{\sumo kn \E S_k S_k^*}, \ve{\sumo kn \E S_k^*S_k}}. 
    \]
    Then 
    \[
\Pr{\ve{Z}\ge t}
\le (d_1+d_2) \exp\pa{
   -\fc{t^2/2}{v(Z) + Lt/3} 
}.
    \]
\end{theorem}


\begin{lemma}[Eigenfunction balance with high probability] \label{lem:vec bound}
Let $\sL$ denote a Markov chain generator with stationary distribution $\pi:\mathcal{D}\to \R_{\geq 0}.$ Let $(f_i)_{i=1}^{\iy}$ be the eigenfunctions of $\sL$ with eigenvalues $-\la_1 = 0> - \la_2\ge -\la_3\ge\cdots$. 
Fix $ k\geq 2$ and recall that $ f_{2:k}(y) = (f_2(y), \dots, f_k(y))\in \R^k.$
 Let $\ep \in (0, \rc{2}]$.
    Let 
    \[ \Om = \set{y\in \mathcal{D} : \ve{f_{2:k}(y)}\le \frac{\sqrt{k-1}}{\ep}}. \] 
    Then $ \pi(\Om^c) \leq \ep^2.$

    Let $y_1,\ldots, y_n\sim \pi$ be iid. 
    Consider $\tilde{\Om}\subseteq \Om$ and $U=\set{y_j: y_j \in \tilde{\Omega}}.$ If $ \pi(\tilde{\Om}^c) \leq 2\ep^2$ then 
    \[\Pr{ \ve{\frac{1}{|U|}\sum_{y_j \in U} f_{2:k}(y_j)} \ge 4 \ep \sqrt{k} } \le k \exp\ba{-\Om (n \ep^2)}.\]
\end{lemma}
\begin{proof}
    Since the eigenfunctions have norm 1, $\E_\pi \ve{f_{2:k}(y)}^2 = \sum_{i=2}^k \E_{y \sim \pi}[f_i(y)^2] = \sum_{i=2}^k ||f_i||_\pi^2 = k-1$.
    By Markov's inequality,
    \begin{align*}
        \pi\pa{\Om^c} 
        < \fc{\E_{y\sim \pi} \ve{f_{2:k}(y)}^2}{(k-1)/\ep^2}
        = \fc{k-1}{(k-1)/\ep^2} = \ep^2.
    \end{align*}

    We now proceed to prove the main claim. 
    Let $\pi' = \pi|_{\tilde{\Om}} = \fc{\pi(\cdot \cap \tilde{\Om})}{\pi(\tilde{\Om})}$.
We claim that $ \ve{\E_{\pi'}[f_{2:k}]}\leq\ep \sqrt{8(k-1)}.$ 
Indeed, 
\begin{align}
\nonumber
  \ve{\E_{\pi'}[f_{2:k}(y)]}&= \rc{\pi(\tilde{\Om})} \ve{ \int_{\tilde{\Om}} f_{2:k}(y) \pi(y) dy } 
  = \rc{\pi(\tilde{\Om})} \ve{\int_{\tilde{\Om}^c}f_{2:k}(y) \pi(y) dy }
\\
\nonumber
&\leq \rc{\pi(\tilde{\Om})}  \sqrt{\int_{\tilde{\Om}^c} \pi(y) dy}\cdot \sqrt{\int_{\tilde{\Om}^c} \pi(y) \ve{f_{2:k}(y)}^2 dy} \\
\label{e:eig-conc-1}
&\leq 2 \sqrt{\pi(\tilde{\Om}^c)} \cdot \sqrt{\E_{\pi}[\ve{f_{2:k}}^2] }\leq  \ep\sqrt{8(k-1)} \leq \sqrt{2(k-1)} 
\end{align}
where the second equality is due to $\E_{\pi}[f_i]= 0$ (by the orthogonality of eigenfunctions),  the first inequality follows from Cauchy-Schwarz, and the second from $ \pi(\tilde{\Om}) \geq 1/2.$ To use Matrix Bernstein, we calculate 
\begin{align}
\nonumber 
    \ve{\E_{\pi'} \vvt{\pa{f_{2:k}-\E_{\pi'} f_{2:k}}}}
    &\le 
    \E_{\pi'} \vtv{\pa{f_{2:k}-\E_{\pi'} f_{2:k}}}\\
    \nonumber
    &\le \E_{\pi'} \ve{f_{2:k}}^2\le \rc{\pi(\wt \Om)} \E_\pi \ve{f_{2:k}}^2\\
    &\le 2 \E_\pi \ve{f_{2:k}}^2 = 2(k-1)\le 2k.
    \label{e:eig-conc-2}
\end{align}
Conditioned on $|U|=m$, $y_j\in U$ are independent draws from $\pi'$. 
Applying \Cref{t:mat-bern} to $f_{2:k}(y_j) -\E_{\pi'} f_{2:k}$ and using~\eqref{e:eig-conc-2} gives us
\begin{align*}
    \Pr{ \ve{\rc{|U|}\sum_{y_j\in U} [f_{2:k}(y_j) -\E_{\pi'} f_{2:k}]} \ge \ep \sqrt{k} \Bigg| |U|=m}
    &\le k \exp\ba{-\Om\pf{m\ep^2 k}{k + \fc{\sqrt k}{\ep}\cdot \ep \sqrt{k} }}\\
    &\le k \exp\ba{-\Om (m\ep^2)}
\end{align*}
Using~\eqref{e:eig-conc-1} ($\ve{\E_{\pi'}[f_{2:k}(y)]}<\ep \sqrt{8(k-1)}$), we obtain
\begin{align*}
    \Pr{ \ve{\rc{|U|}\sum_{y_j\in |U|} f_{2:k}(y_j)} \ge 4 \ep \sqrt{k} \Bigg| |U|=m} &\le k \exp\ba{-\Om (m\ep^2) }. 
\end{align*}
Finally, we note that by Hoeffding's inequality 
\begin{align*}
    \Pr{|U| \le \rc 4 n}
    &\le \exp\pa{-\Om(n)}.
\end{align*}
Taking a union bound and adjusting constants appropriately then gives the result.
\end{proof}
\subsection{Proof of main theorem}
\begin{proof}[Proof of \cref{thm:main}]

We prove the first statement.
Let $\ep =\frac{\etv}{8 \sqrt{k}}<1/2.$  
As in \cref{lem:vec bound}, define
\[ \Omega =\set{y\in \mathcal{D} : \ve{f_{2:k}(y)}\le \frac{\sqrt{2(k-1)}}{\ep}}, \]
let $ \tilde{\Om} = \Om \cap \Ombd$, and let $U = U_{\text{sample}} \cap \tilde{\Om}.$ By \cref{lem:vec bound}, $ \pi(\Om^c) \leq \ep^2,$ thus $ \pi(\tilde{\Om}^c) \leq 2 \ep^2 = \fc{\etv^2}{32k}$ by a union bound and the definition of $ \Ombd.$


Let $\mathcal{E}_1$ be the event that
\[ \ve{\frac{1}{|U|}\sum_{y_j \in U} f_{2:k}(y_j)} \le  4 \ep \sqrt{k}, \] 
which happens with probability at least $ 1 - k \exp(-\Om(n \ep^2))$ by \Cref{lem:vec bound}.

Let  $\mu'_t$ be the distribution at time $t$ of 
Langevin diffusion initialized at $\mu'_0= \frac{1}{\ab{U}} \sum_{y_j\in U}\delta_{y_j}.$ 
By \cref{lem:weak convexity} and the definition of $U,$ for 
\[ \chi^2(\mu'_{t_0}||\pi) \leq \max_{y_j\in U} \chi^2(\rho^{y_j}_{t_0}||\pi) \leq R. \] 
By~\Cref{l:small-sum}, 
if $ \mathcal{E}_1$ holds, then for $t\ge t_0$, 
\[
\chi^2(\mu_t' \|\pi) \le 16\ep^2 k + e^{-\al (t-t_0)}R
= \fc{\etv^2}{4} + e^{-\al(t-t_0)}R.
\]

We can rewrite $ \mu_0 = (1-p) \mu'_0+ p\mu''_0$ where $\mu'_0 =\sum_{y_j\in U} \delta_{y_j}$, $ \mu''_0=\sum_{y_j\not\in U} \delta_{y_j}$ and 
$p=\fc{|U_{\textup{sample}}\bs U|}n$. Let $\mathcal{E}_2$ be the event $p\leq 4 \ep^2$. 
Since $\pi(\tilde{\Omega}^c)\leq 2 \ep^2$, by  Chernoff's bound, 
\[\Pr{\mathcal{E}_2} \geq 1- \exp(-\Omega(n \ep^2)) .\]

Let  $\mu_t, \mu_t', \mu''_t$ be the distribution at time $t$ of the process driven by $P_t$ initialized at $\mu_0, \mu_0', \nu_0$, respectively; then $\mu_t = (1-p) \mu'_t + p\mu''_t.$ With probability $ 1- k\exp(-\Omega (n \ep^2 )) $ over the random choice of the samples $y_j,$ both  $\mathcal{E}_1$ and $\mathcal{E}_2$ hold, and thus 
by triangle inequality (for details, see \cite[Proposition 9]{koehler2023sampling}) 
\begin{align*}
\TV (\mu_t, \pi) &\leq (1-p) \TV  (\mu'_t, \pi) + p \TV (\mu''_t, \pi) \\
&\leq 
\sqrt{\fc{\etv^2}4 + e^{-\al (t-t_0)}R} + p
\le \sqrt{\fc{\etv^2}4 + e^{-\al (t-t_0)}R} + \fc{\etv}{16k}.
\end{align*}
When $t\ge t_0+\fc1\al \ln \pf{2R}{\etv^2}$, this is $\le \etv$. When $n=\Om\pa{\fc{k}{\etv^2}\ln \pf k\de}$, the probability of $\cal E_1\cup \cal E_2$ is $\ge 1-\de$.
\end{proof}

\subsection{Error in sampled distribution}
\label{s:error-sampled}
We will also need the following generalization of \Cref{thm:main}, which shows it is robust to TV error in the sampled distribution.
\begin{theorem}\label{thm:generalize main}
    Keep the setting of \Cref{thm:main}, but assume that $y_1,\ldots, y_n$ are drawn iid from $\nu$. Then for $n = \Om\pa{\fc{k}{\etv^2} \ln \pf{k}{\de}}$, 
$t\ge T:=t_0+\rc \al \ln \pf{4R}{\etv^2}$ and $\TV(\pi,\nu) \le \fc{\etv}{16}$, with probability $\ge1-\de$ over the randomness of the sample, 
    \[
\TV(\mu_t, \pi) \le \etv.
    \]
\end{theorem}
\begin{proof}
    We use a coupling argument. Let $y_1^\pi,\ldots, y_n^\pi\sim \pi$ be iid, and couple these variables with $y_1,\ldots, y_n$ such that $\Pj(y_i^\pi\ne y_i) = \TV(\mu,\nu)$. By a Chernoff bound,
    \[
\Pj\pa{\ab{\bc{i:y_i^\pi \ne y_i}}\le \fc{\etv}8 n}
\ge 1-e^{-\Om(\etv n)}.
    \]
    Thus for $n=\Om\pa{\rc{\etv}\ln \pf 1\de}$, this is $\ge 1-\fc\de2$.
    Under this event, by triangle inequality
    \[
\TV\pa{\mu_T, \rc n \sumo in \de_{y_i}^\pi P_T} =\TV\pa{\rc n \sumo in \de_{y_i} P_T, \rc n \sumo in \de_{y_i}^\pi P_T}  \le \fc{\etv}8 .
    \]
    Also with probability $\ge 1-\fc\de2$, by \Cref{thm:main},
    \[
\TV\pa{ \rc n \sumo in \de_{y_i}^\pi P_T , \pi}\le
\sqrt{\fc{\etv^2}4 + e^{-\al (t-t_0)}R} + \fc{\etv}{16k}\le \fc{7\etv}8.
    \]
    The theorem follows from a union bound and triangle inequality.
\end{proof}

\section{Markov chain perturbation}\label{s:perturb}

We show that for both Langevin and Glauber dynamics, if the dynamics are perturbed within some error, and the chain is started at the original stationary distribution $\pi$, then the KL divergence between $\pi$ and the resulting distribution grows at most linearly in this perturbation. By an averaging and concentration argument, this then implies that for most sets of samples drawn from $\pi$, the distribution starting from those samples also stays close. 

\subsection{Error growth when starting from stationary distribution}
For Langevin dynamics, the appropriate notion of error is the $L^2$ error in the score estimate, and we bound the growth in error using Girsanov's Theorem, similar to the analysis for reverse diffusion in \cite{chen2023sampling}.
\begin{remark}[Connection to score matching]
A $L^2$-accurate score function can naturally be obtained from score matching; in particular, \cite{hyvarinen2005estimation} showed via integration by parts that
\[  \frac{1}{2} \E_{X \sim p}[\|\nabla \log p(X) - \nabla \log q(X)\|^2] + K_p = \E_{X \sim p}\left[\Tr \nabla^2 \log q + \frac{1}{2} \|\nabla \log q\|^2\right] \]
where $K_p$ is a constant independent of $p$. If we consider the empirical analogue of the right hand side, then this gives Hyv\"arinen's score matching objective which can be optimized over a finite dataset, and standard tools from statistical learning theory can be used to recover bounds on the population score matching loss, i.e. the $L^2$-error of the score estimate (see e.g. proof of Theorem 1 of \cite{koehler2022statistical}).
\end{remark}

We let $\cL{}{s}{t}$ denote Langevin diffusion with score estimate $s$, satisfying the SDE
\[
d\cL{}{s}{t}= 
s(\cL{}{s}{t}) \,dt + 
\sqrt2 \,dB_t.
\]
This reduces to the usual Langevin diffusion with stationary distribution $\pi\propto e^{-V}$ when $s = -\nb V$.
We denote the discretized Langevin dynamics with score estimate $s$ and  given step size $h$ by $\dL{}{s}{t}$, and extend it to continuous time by interpolation:
\[
d\dL{}{s}{t} = 
s( \dL{}{s}{\fl{t/h}h})\,dt
+ \sqrt2 \,dB_t.
\]
In both cases, we denote the initial distribution (the distribution at $t=0$) as a superscript.

\begin{lemma}
\label{l:ld-score-error}
Let $\pi\propto e^{-V}$ where $V$ is $\be$-smooth. 
Let $s$ be such that $\E_{\pi}\ve{s - (-\nb V)}^2\le \esc^2$. 
Then we have
\begin{align*}
    \KL(\cal L ((\cL {\pi}{-\nb V}{t})_{0\le t\le T})\| 
        \cal L ((\cL {\pi}{s}{t})_{0\le t\le T}))
        &\le T\esc^2\\
    \KL(\cal L ((\cL {\pi}{-\nb V}{t})_{0\le t\le T})\| 
        \cal L ((\dL {\pi}{s}{t})_{0\le t\le T})) &\le 8 T (  h^2 d\be + h d ) \be^2 + 2T\esc^2 \\
        &\quad \text{ when }T=Nh, \, N\in \N.
        \qedhere 
\end{align*}
\end{lemma}
\begin{proof}
    By Girsanov's Theorem (we do not need to check Novikov's condition since we can use \cite[Eq (5.5) and Theorem 9]{chen2023sampling}), we have that
    \begin{align*}
        \KL(\cal L ((\cL {\pi}{-\nb V}{t})_{0\le t\le T})\| 
        \cal L ((\cL {\pi}{s}{t})_{0\le t\le T}))
        &\le \int_0^T \E \ve{\nb \ln \pi(\cL{\pi}{-\nb V}{t}) - s(\cL{\pi}{-\nb V}{t})}^2\,dt \le T\esc^2
    \end{align*}
    since the distribution of $\cL{\pi}{-\nb V}{t}$ is $\pi$.
    Let $t_k = kh$. 
    Write $\cL{\pi}{-\nb V}{t}$ as $\bar{X}_t$ for short. 
    Also by Girsanov's Theorem,
    \begin{align}
    \nonumber
        &\KL(\cal L ((\cL {\pi}{-\nb V}{t})_{0\le t\le T})\| 
        \cal L ((\dL {\pi}{s}{t})_{0\le t\le T})) \le 
        \sumz k{N-1}\int_{t_{k}}^{t_{k+1}}
        \E \ve{\nb \ln \pi(\bar{X}_{t}) - s(\bar{X}_{t_k})}^2\,dt\\
    \label{e:compare-g0}
        &\le 2\sumz k{N-1} \int_{t_k}^{t_{k+1}} \E \ve{\nb \ln \pi(\bar{X}_{t}) - \nb \ln \pi(\bar{X}_{t_k})}^2 + 
        \ve{\nb \ln \pi(\bar{X}_{t_k}) - s(\bar{X}_{t_k})}^2 \,dt\\
        &\le 2\pa{\sumz k{N-1} \int_{t_{k}}^{t_{k+1}} \be^2 \E \ve{\bar{X}_{t}-\bar{X}_{t_k}}^2\,dt} + 2T\esc^2.
        \label{e:compare-g1}
    \end{align}
    Now 
    \begin{align}
    \nonumber
        \E \ve{\bar{X}_{t}-\bar{X}_{t_k}}^2 
        & = \ve{\int_{t_k}^t -\nb V(\bar{X}_s)\,ds + \int_{t_k}^t \sqrt 2 \,dB_t}^2\\
    \nonumber
        &\le 2 \E \ve{\int_{t_k}^t \nb V(\bar{X}_s)\,ds}^2 + 4(t-t_k)d \\
    \nonumber 
        &\le 2(t-t_k)^2 \E \ve{\nb V(\bar{X}_s)}^2 + 4(t-t_k)d \\
    \nonumber 
        &\le 2(t-t_k)^2\be d + 4 (t-t_k) d 
    \end{align}
    using the bound $\E_\pi \ve{-\nb V}^2 \le \be d$ from \Cref{l:norm bound}.
    Hence
    \begin{align}
    \int_{t_{k}}^{t_{k+1}} \E \ve{\bar{X}_{t}-\bar{X}_{t_k}}^2 \,dt &\le 2h^3 \be d + 4h^2 d.
    \label{e:compare-g2}
    \end{align}
    Combining \eqref{e:compare-g1} and \eqref{e:compare-g2} gives
    \begin{align*}
        \KL(\cal L ((\cL {\pi}{-\nb V}{t})_{0\le t\le T})\| 
        \cal L ((\cL {\pi}{s}{t})_{0\le t\le T}))
        &\le 
        2\cdot \fc Th \cdot h\cdot( 2 h^2 d\be + 4h d ) \be^2 + 2T\esc^2.
        \qedhere
    \end{align*}
\end{proof}
When $ \pi$ is not smooth but is a mixture of smooth distributions $\pi_i$ with good tail bounds, we have a similar result.

\begin{lemma}\label{l:ld-score-error-mix}
Let $\pi\propto e^{-V} =\sum_i p_i \pi_i$ 
where each $\pi_i=\exp(-V_i)$ is $\beta$-smooth, $p_i>0$, and $\sum_i p_i=1$. 
Let $ G(x)=\max_i \ve{\nb V_i(x)}.$ Suppose  $ \E_{\pi} [G(x)^6] \leq \tilde{G}^6.$

Let $s$ be such that $\E_{\pi}\ve{s - (-\nb V)}^2\le \esc^2$. 
Then we have
\begin{align*}
    \KL(\cal L ((\cL {\pi}{-\nb V}{t})_{0\le t\le T})\| 
        \cal L ((\cL {\pi}{s}{t})_{0\le t\le T}))
        &\le T\esc^2\\
    \KL(\cal L ((\cL {\pi}{-\nb V}{t})_{0\le t\le T})\| 
        \cal L ((\dL {\pi}{s}{t})_{0\le t\le T})) &
        \lesssim T\cdot
        \pa{\be^4 (hd + h^6\tilde G^6 + h^3 d^3) + (h^2 \tilde G^2 + hd) \tilde G^4 + \esc^2}\\
        &\quad 
        \text{ when }T=Nh, \, N\in \N.
\end{align*}
\end{lemma}
\begin{proof}
    Note that since $ \ve{\nb V(x)}\leq G(x),$ for $p\leq 6$, $\E_{\pi}[\ve{\nb V(x)}^p]\leq \E_{\pi}[G(x)^p] \leq \tilde{G}^p$ by H\"older's inequality.
    Let $\bar{X}_t:= \cL{\pi}{-\nb V}{t}$.
    It suffices to bound \eqref{e:compare-g0}.
    By the mean value inequality, 
    \[ \ve{\nb V(\bar{X}_{t}) - \nb V(\bar{X}_{t_k})}\leq \ve{\bar{X}_t - \bar{X}_{t_k}} \max_{y = \eta \bar{X}_t + (1-\eta) \bar X_{t_k}, \eta\in [0,1] } \norm{\na^2 V(y)}_{\textup{op}}.\]
By \cref{l:hessian bound for mixture} and H\"older's inequality,
\[\norm{\na^2 V(y)}_{\textup{op}} \leq G(\bar{X}_t)^2 +G (\bar{X}_{t_k})^2 +\beta. \]
 For $y=\eta \bar X_t + (1-\eta) X_{t_k}$, $\eta\in [0,1]$, 
\begin{align*}
    \ve{\nb^2 V(y)}_{\textup{op}}
    &\le \be + G(y)^2\\
    &\le \be + (G(\bar X_t) + \be \ve{y-\bar X_t})^2\\
    &\le \be + 2G(\bar X_t)^2 + 2\be^2 \ve{\bar X_{t_k} - \bar X_t}^2
\end{align*}
Then,
\begin{align*}
    \ve{\nb V(\bar X_t) - \nb V(\bar X_{t_k})}^2
    &\le 
    \ve{\bar X_{t_k}-\bar X_t}^2
    \pa{\be^2 + 2G(\bar X_t)^2 + 2\be^2 \ve{\bar X_{t_k} - \bar X_t}^2}^2\\
    &\le 
    3\ve{\bar X_{t_k}-\bar X_t}^2
    \pa{
    \be^4 + 4G(\bar X_t)^4 + 4\be^4 \ve{\bar X_{t_k} - \bar X_t}^4
    }
\end{align*}
Hence
\begin{align*}
&\E\ve{\nb V(\bar X_t) - \nb V(\bar X_{t_k})}^2\\
        &\le
    3\be^4 \E[\ve{\bar X_{t_k}-\bar X_t}^2] + 12\be^4 \E[\ve{\bar X_{t_k}-\bar X_t}^6]
    + 12 \E[\ve{\bar X_{t_k}-\bar X_t}^2G(\bar X_t)^4] \\
    &\le 
    3\be^4 \E[\ve{\bar X_{t_k}-\bar X_t}^2] + 12\be^4 \E[\ve{\bar X_{t_k}-\bar X_t}^6]
    +
    12 \E[\ve{\bar X_{t_k}-\bar X_t}^6]^{1/3} \tilde G^4
\end{align*}
where the last step uses H\"older's inequality.

 We bound for $p\ge 1$
  \begin{align}
    \nonumber
        \E \ve{\bar{X}_{t}-\bar{X}_{t_k}}^{2p} 
        & = \E \left[\ve{\int_{t_k}^t -\nb V(\bar{X}_s)\,ds + \int_{t_k}^t \sqrt 2 \,dB_t}^{2p}\right ]\\
    \nonumber
        &\le 2^{2p-1} \E \left[\ve{\int_{t_k}^t \nb V(\bar{X}_s)\,ds}^{2p}\right] + 2^{3p-1} (t-t_k)^p d^p \\
    \nonumber 
        &\le 2^{2 p-1}(t-t_k)^{2p} \E \ve{\nb V(\bar{X}_s)}^{2p} + 2^{3p-1}(t-t_k)^{p} d^p \\
    \nonumber 
        &\leq 2^{2p-1} (t-t_k)^{2p} \tilde{G}^{2p} +  2^{3p-1} (t-t_k)^p d^p
    \end{align} 
    For $p=1,$ using the bound  $\E{\ve{\nb  V}^2}\leq \beta d$ from \Cref{l:norm bound} gives
    \[\E\ve{\bar{X}_t -\bar{X}_{t_k}}^2 \leq 2 (t-t_k)^2 \beta d + 4 (t-t_k) d \]
    Hence,
    \begin{align*}
        \E\ve{\nb V(\bar X_t) - \nb V(\bar X_{t_k})}^2
        &\lesssim 
        \be^4 (hd + h^6\tilde G^6 + h^3 d^3) + (h^2 \tilde G^2 + hd) \tilde G^4.
    \end{align*}
    Finally, substituting into \eqref{e:compare-g0} gives
   the desired bound.
\end{proof}
Next we show a similar error bound for the Gibbs sampler (Glauber dynamics) that if it is executed with a distribution $\mu$ which approximately matches the conditional law of $\nu$. This is useful because pseudolikelihood estimation can naturally produce such a $\mu$ from samples. 

\begin{lemma}\label{lem:gibbs-approximate}
Suppose $t \ge 0$, $\ep  > 0$ and that $\nu$ and $\mu$ are distributions on a product space $\Sigma_1 \otimes \cdots \otimes \Sigma_n$ such that
\begin{equation}\label{eqn:eps-accurate-gibbs}
\frac{1}{n} \sum_{i = 1}^n \E_{X \sim \nu} \KL(\nu(X_i = \cdot \mid X_{\sim i}), \mu(X_i = \cdot \mid X_{\sim i})) \le \ep . 
\end{equation}
Let $X^{(0)},\ldots,X^{(t)}$ is the trajectory of the $\nu$-Gibbs sampler initialized from $X^{(0)} \sim \nu$, and let $\tilde X^{(0)}, \ldots, \tilde X^{(t)}$ be the trajectory of the $\mu$-Gibbs sampler with the same initialization $\tilde X^{(0)} = X^{(0)}$. Then the KL divergence between the law of the trajectories can be bounded as
\begin{equation} \KL(\mathcal{L}(X^{(0)}, \ldots, X^{(t)}), \mathcal{L}(\tilde X^{(0)}, \ldots, \tilde X^{(t)})) \le t \ep . 
\end{equation}
Furthermore, for any $\delta > 0$, for the trajectories conditioned on the initial point $X^{(0)} \sim \nu$, we have
\begin{equation} \nu\left(\KL(\mathcal{L}(X^{(0)}, \ldots, X^{(t)} \mid X^{(0)}), \mathcal{L}(\tilde X^{(0)}, \ldots, \tilde X^{(t)} \mid \tilde X^{(0)} = X^{(0)}) > t \ep /\delta \right) \le \delta. \label{eqn:good-event-markov}
\end{equation}
\end{lemma}
\begin{proof}
By the chain rule for the KL divergence (see e.g. \cite{cover1999elements}) and the fact that the dynamics are Markovian, we have
\begin{align*} 
&\KL(\mathcal{L}(X^{(0)}, \ldots, X^{(t)}), \mathcal{L}(\tilde X^{(0)}, \ldots, \tilde X^{(t)})) \\
&= \sum_{s = 1}^t \E_{X^{(s - 1)} \sim \nu} \KL(\nu(X^{(s)} = \cdot \mid X^{(s - 1)}), \mu(\tilde X^{(s)} = \cdot \mid \tilde X^{(s - 1)} = X^{(s - 1)}) \\
&\le t \ep 
\end{align*}
where we used the assumption and the fact that a step of the Gibbs sampler picks a coordinate $i \sim Uni [n]$ and then samples from the corresponding conditional law. 

The second claim in the theorem follows by applying Markov's inequality, using by the chain rule that 
\begin{align*} 
&\KL(\mathcal{L}(X^{(0)}, \ldots, X^{(t)}), \mathcal{L}(\tilde X^{(0)}, \ldots, \tilde X^{(t)})) \\
&= \E_{X^{(0)} \sim \nu} \KL(\mathcal{L}(X^{(0)}, \ldots, X^{(t)} \mid X^{(0)}), \mathcal{L}(\tilde X^{(0)}, \ldots, \tilde X^{(t)} \mid \tilde X^{(0)} = X^{(0)})). \qedhere 
\end{align*}
\end{proof}

\subsection{Error growth when starting from samples}


\begin{theorem}\label{thm:approximate-sample-init}
    Suppose that we have a bound 
    \[
\KL(\cal L((X^\pi_t)_{0\le t\le T}) \| \cal L((\tilde X^\pi_t)_{0\le t\le T})) \le T\ep\le 1,
    \]
    where $X^\pi_t$ and $\tilde X^\pi_t$ denote Markov chains or processes $X_t$ and $\tilde X_t$ initialized at $\pi$ (the index set can be $\N_0$ or $\R_{\ge 0}$). 

    Let $\hat \pi$ be the empirical distribution of $m$ i.i.d. samples from $\pi$.
    Then with probability at least $1 - \gamma$ over the randomness of $\hat \pi$, the law of the trajectories initialized at $\hat \pi$ satisfy
\[ \TV (\mathcal L((X^{\hat \pi}_t)_{0\le t\le T}\mid \hat \pi), \mathcal L((\tilde X^{\hat \pi}_t)_{0\le t\le T} \mid \hat \pi)) \le \sqrt{T\ep} + \fc{2\log(1/\gamma)}{m}. \] 
\end{theorem}
\begin{proof}
By convexity of TV distance, 
\begin{align*}
    \TV (\mathcal L((X^{\hat \pi}_t)_{0\le t\le T}\mid \hat \pi), \mathcal L((\tilde X^{\hat \pi}_t)_{0\le t\le T} \mid \hat \pi)) 
    \le 
    \rc m \sumo im \TV  
    (\mathcal L((X^{Y_i}_t)_{0\le t\le T}\mid Y_i), \mathcal L((\tilde X^{Y_i}_t)_{0\le t\le T} \mid Y_i)) 
\end{align*}
where $\hat \pi = \rc m \sumo im \de_{Y_i}$, $Y_i$ being independent draws from $\pi$. Now for $Y\sim \pi$, 
\begin{align*}
    &\Var_{Y\sim \pi}(\TV  
    (\mathcal L((X^{Y}_t)_{0\le t\le T}), \mathcal L((\tilde X^{Y}_t)_{0\le t\le T})))\\
    &\le \E_{Y\sim \pi} [\TV  
    (\mathcal L((X^{Y}_t)_{0\le t\le T}), \mathcal L((\tilde X^{Y}_t)_{0\le t\le T}))^2]\\
    &\le \rc 2 \E_{Y\sim \pi} [\KL 
    (\mathcal L((X^{Y}_t)_{0\le t\le T}) \| \mathcal L((\tilde X^{Y}_t)_{0\le t\le T}))] &\text{by Pinsker's inequality}\\
    &= \rc 2 \KL(\mathcal L((X^{\pi}_t)_{0\le t\le T})\| \mathcal L((\tilde X^{\pi}_t)_{0\le t\le T})) \le \rc 2 T\ep 
    &\text{by chain rule for KL}.
\end{align*}
This calculation also shows, by  Jensen's inequality
\begin{multline*}
\E_{Y\sim \pi} [\TV  
    (\mathcal L((X^{Y}_t)_{0\le t\le T}), \mathcal L((\tilde X^{Y}_t)_{0\le t\le T}))]\\
\le \sqrt{\E_{Y\sim \pi} [\TV  
    (\mathcal L((X^{Y}_t)_{0\le t\le T}), \mathcal L((\tilde X^{Y}_t)_{0\le t\le T}))^2]}
    \le \sqrt{\rc 2 T\ep}.
\end{multline*}
Since TV distance is bounded by 1, by Bernstein's inequality,
\begin{align*}
\Pj\pa{
    \rc m \sumo im \TV  
    (\mathcal L((X^{Y_i}_t)_{0\le t\le T}\mid Y_i), \mathcal L((\tilde X^{Y_i}_t)_{0\le t\le T} \mid Y_i))
 \ge \sqrt{\rc 2 T\ep} + u
}&\le  \exp\ba{-\fc{mu^2/2}{\rc 2 T\ep + u/3}}.
\end{align*}
Taking $u = \max\bc{\sfc{2T\ep \ln (1/\ga)}{m}, \fc{4\ln (1/\ga)}{3m}}$, we get that this is $\le \ga$.
Hence, with probability $\ge 1-\ga$, 
\[
\TV (\mathcal L((X^{\hat \pi}_t)_{0\le t\le T}\mid \hat \pi), \mathcal L((\tilde X^{\hat \pi}_t)_{0\le t\le T} \mid \hat \pi)) \le \sqrt{\rc 2 T\ep} + 
\max\bc{\sfc{2T\ep \ln (1/\ga)}{m}, \fc{4\ln (1/\ga)}{3m}}.
\]
Finally, using the inequality $2ab\le a^2+b^2$, 
\[
\sfc{2T\ep \ln (1/\ga)}{m} 
\le \rc 4 T\ep + \fc{2\ln (1/\ga)}{m}
\le \rc 4 \sqrt{T\ep} + \fc{2\ln (1/\ga)}{m}
\]
giving the desired conclusion.
\end{proof}

\section{Application: Langevin dynamics with estimated score}\label{s:langevin}
In this section, we illustrate how to apply our general results to sample a mixture of Poincar\'e distributions with an approximate score function. We also show how to obtain stronger results when the components satisfy the stronger log-Sobolev inequality. 

\begin{assumption}[Mixture assumption]\label{assumption:mixture}
Let $\pi = \sumo ik p_i\pi_i$ be a mixture of distributions $\pi_i \propto e^{-V_i(x)}$, where $p_i>0$ for each $i$ and $\sumo ik p_i=1$. Suppose each 
$V_i$ is $\be$-smooth and satisfies either:
\begin{itemize}
    \item $\mathsf{PI}\prc\al$: A Poincar\'e inequality with constant $\CP\le \rc \al$.
    \item $\mathsf{LSI}\prc\al$: A log-Sobolev inequality with constant $\CLS\le \rc \al$.
\end{itemize}
We let $\ka:=\fc{\be}{\al}$. 
Let $\mean_i:= \E_{\pi_i}x$ and $\mode_i\in \amin V_i$. Suppose $L$ is such that the means satisfy
\[\max_{i,j}\ve{\mean_i - \mean_j}\le L.\]
\end{assumption}
Note that we assume a bound on the distance between means rather than modes, because it is more natural to establish concentration around the mean. 

Given \Cref{assumption:mixture}, we observe that 
from \Cref{l:pi-conc} in \cref{s:fi-conseq}, we obtain the concentration bound
\begin{align}\label{e:conc}
\Pj\pa{\min_{i} \ve{x-\E_{\pi_i}x} \ge R_\ep} 
\le \ep, \quad \text{where }R_\ep:= \sqrt{\rc\al} \pa{\sqrt d + \ln \pf 3\ep}.
\end{align}

\subsection{Warm start}
\label{s:warm}
To apply our result for the Langevin dynamics initialized from samples, we need to control the R\'enyi distance to stationarity of  the distribution of the Langevin diffusion initialized at a typical sample $x.$
\begin{lemma}\label{l:riy}
    Let $\nu$ be the measure of $\cal N(y,\si^2 I)$. If $\pi(x) \propto \exp(-V(x))$ is $\be$-smooth and satisfies a Poincar\'e inequality with constant $\CP\le \rc \al$ and $\si^2\be \le \rc 2$, 
    then 
    \[
\Renyi_{\iy}(\nu \|\pi)\le 
\si^2\ve{\nb V(y)}^2 + V(y) - V(x^*) +1 - \rc 2 \ln (\pi d) + \fc{d}{2}\ln \pf{e}{\al \si^2}.
    \]
\end{lemma}
\begin{proof}
We use $\be$-smoothness of $V$ to upper bound $e^{V(x)}$ and \Cref{l:Z-bound} to upper bound $\int_{\R^d} e^{-V(x)}\dx$.
\begin{align*}
    \fc{\nu(x)}{\pi(x)} 
    &= (2\pi \si^2)^{-d/2}
        e^{-\fc{\ve{x-y}^2}{2\si^2} + V(x)} \cdot \int_{\R^d} e^{-V(x)}\dx\\
    &\le (2\pi \si^2)^{-d/2}
        e^{-\fc{\ve{x-y}^2}{2\si^2} + V(y) + \an{\nb V(y), x-y} + \fc\be2\ve{x-y}^2}
        e^{-V(x^*)}\fc{e}{\sqrt{\pi d}} \pf{2e\pi}{\al}^{d/2}\\
    &= \exp\ba{-\rc 2\pa{\rc{\si^2}-\be} \ve{x-y - \rc{\rc{\si^2}-\be}\nb V(y)}^2 + \rc{2\pa{\rc{\si^2}-\be}} \ve{\nb V(y)}^2 + V(y)-V(x^*)} \\
    &\quad \cdot 
    \fc{e}{\sqrt{\pi d}}\pf{e}{\al \si^2}^{d/2}\\
    &\le \exp\ba{\fc{\si^2}{2(1-\si^2\be)} \ve{\nb V(y)}^2 + V(y) - V(x^*)}\fc{e}{\sqrt{\pi d}}\pf{e}{\al\si^2}^{d/2}\\
    &\le \exp\ba{\si^2 \ve{\nb V(y)}^2 + V(y) - V(x^*)}\fc{e}{\sqrt{\pi d}}\pf{e}{\al \si^2}^{d/2}
\end{align*}
Taking the log gives the result.
\end{proof}

\begin{lemma}\label{lem:cont-init2}
Suppose $\pi$ satisfies \Cref{assumption:mixture} with the Poincar\'e inequality $\PI\prc{\al}$. 
Let $ \bar{\nu}_t, \nu_t$ be respectively the distribution at time $t$ of the continuous Langevin diffusion and the LMC with step size $h$ 
initialized at $\delta_x$. Let $G(x):=\max_i \norm{\nabla V_i(x)}$. 
Suppose $h \leq 1/(30\beta) $. Then 
\begin{align*}
   \Renyi_q  (\bar{\nu}_h || \nu_h)
&\lesssim q^2 h (G^2(x) + \beta^2 dh)
&\text{for any $q\in \pa{1, \frac{1}{10\beta h }}$}
\\
\Renyi_{q} (\nu_h || \pi) &\lesssim h \max_i\ve{\nb V_i(x)}^2 + \max_i (V_i(x) - V_i(\mode_i)) + d\pa{1+\ln \prc{\al h}}
&\text{for any $q\in (1, \iy]$}
\end{align*}
and
\begin{align*}\Renyi_{q}(\bar{\nu}_h ||  \pi) &\lesssim
q^4 h (G^2(x) + \beta^2 dh) + h\max_i \ve{\nb V_i(x)}^2\\
& \quad + \max_i(V_i(x) - V(\mode_i)) + d\pa{1+\ln \prc{\al h}}
&
\text{for any $q\in \pa{1, \sfc{1}{10\beta h }}$}.
\end{align*}
\end{lemma}
\begin{proof}
The first inequality follows from \cite[Lemma 10]{koehler2023sampling} ($\al$-strong log-concavity is not used in the proof). 
    For the second inequality, note by weak-convexity of Renyi divergence (\Cref{lem:weak convexity}) and \Cref{l:riy}, 
    \begin{align*}
        \Renyi_{q} (\nu_h || \mu) &
        \le \max_i \Renyi_{q} (\nu_h || \mu_i)
        \le \max_i \Renyi_{\iy} (\nu_h || \mu_i)\\
        &\lesssim \max_i \left (h \ve{\nb V_i(x-h\nb V(x))}^2 + V_i(x-h\nb V(x))-V_i(x^*) + d\pa{1 + \ln \pf{1}{\al h}} \right).
    \end{align*}
    We have by using $h \le \rc{30\be}$ and $\ve{\nb V(x)}\le \max_j \ve{\nb V_j(x)}$ that 
    \begin{align*}
        \ve{\nb V_i(x-h \nb V(x))}^2 &\lesssim \ve{\nb V_i(x)}^2 + \be^2 h^2  \ve{\nb V(x)}^2 \lesssim \max_j \ve{\nb V_j(x)}^2\\
        V_i(x-h\nb V(x)) & \le V_i(x) - h\an{\nb V_i(x), \nb V(x)} + \fc{\be}2h^2 \ve{\nb V(x)}^2\\
        &\le V_i(x) + \pa{h+\fc{\be h^2}2}\max_j \ve{\nb V_j}^2\le V_i(x) + O\pa{h \max_j \ve{\nb V_j(x)}^2}
    \end{align*}
    so
    \begin{align*}
        \Renyi_{q} (\nu_h || \mu)
        &\lesssim h \max_i\ve{\nb V_i(x)}^2 + \max_i (V_i(x) - V_i(\mode_i)) + d\pa{1+\ln \prc{\al h}}.
    \end{align*}
    By the weak triangle inequality \Cref{lem:weak triangle inequality}, choosing $a=q$, $b=\fc{q}{q-1}$,
    \begin{align*}
        \Renyi_q(\bar{\nu}_h ||  \mu)
        &\le
        \fc{aq-1}{a(q-1)} \Renyi_{aq}(\bar{\nu}_h || \nu_h) 
        + \Renyi_{b(q-1)+1} (\nu_h || \mu)\\
        &\lesssim \Renyi_{q^2}(\bar{\nu}_h || \nu_h) 
        + \Renyi_{q+1} (\nu_h || \mu).
    \end{align*}
    Substituting in the previous two inequalities then gives the last bound.
\end{proof}

\begin{corollary}\label{c:R2-bd}
Suppose $\pi$ satisfies \Cref{assumption:mixture} with the Poincar\'e inequality $\PI\prc{\al}$. 
    Let $\ep_1\le \rc 2$, and let $R_\ep = \rc{\sqrt{\al}}\pa{\sqrt d + \ln \pf 3\ep}$. Then 
    \[
\Pj_{x\sim \pi}\pa{\max_j \ve{x-\ol x_j}< R_{\ep_1} + L} \ge 1-\ep_1,
    \]
    and under this event, for 
    $h=\rc{50\be}$, when $\bar \nu_0=\de_x$, for any $q\in (1,2]$, 
    \[
\Renyi_q(\bar \nu_h \|\pi )\lesssim 
\ka d + \ka \ln \prc{\ep_1}^2 + \be L^2 + d\ln\ka. 
    \]
\end{corollary}
\begin{proof}
    By \Cref{l:pi-conc},
    \[
\Pj_{x\sim \pi_i}(\ve{x-\ol x_i} \ge R_{\ep_1})\le \ep_1.
    \]
    Hence, considering a draw over the mixture distribution,
    \[
\Pj_{x\sim \pi}(\exists i, \, \ve{x-\ol x_i} < R_{\ep_1}) \ge 1-\ep_1.
    \]
    Under this event, for all $j$,
    \begin{align*}
        \ve{x-\ol x_j} &< R_{\ep_1} + L\\
        V_j(x) - V_j(\ol x_j)
        &\le \an{\nb V_j(\ol x_j), x-\ol x_j} + \fc \be2 \ve{x - \ol x_j}^2\\
        &\lesssim \be \sqrt{d/\al} (R_{\ep_1}+L) + \be (R_{\ep_1}+L)^2 & \text{by \Cref{l:dV-mean}}\\
        &\lesssim 
        \fc{\be d}{\al} + \be (R_{\ep_1}+L)^2\\
        &\lesssim \ka d + \ka \ln \prc{\ep_1}^2 + \be L^2\\
        V_j(x) - \min V_j &\lesssim \ka d + \ka \ln \prc{\ep_1}^2 + \be L^2 & \text{by \Cref{l:V-mean}}\\
        \ve{\nb V_j(x)} &\le \ve{\nb V_j(x) - \nb V_j(\ol x_j)} + \ve{\nb V_j(\ol x_j)}\\
        &\lesssim \be (R_{\ep_1}+L) + \be \sqrt{d/\al}
        & \text{by \Cref{l:dV-mean}}\\
        &\lesssim \sqrt{\be \ka d} + \sqrt{\be \ka}\ln \prc{\ep_1} + \be L.
    \end{align*}
    Then for $\bar \nu_0=\de_x$, by \Cref{lem:cont-init2}, for $h<\rc{40\be}$,
    \begin{align*}
\Renyi_q(\bar \nu_h \|\pi )
&\lesssim 
h(G^2(x) + \be^2 dh) + h\max_j\ve{\nb V_j(x)}^2 \\
&\quad + \max_j (V_j(x) - V_j(x_i^*))  + d\pa{1+\ln \prc{\al h}}
    \end{align*}
    Choose $h=\rc{50\be}$. 
By the above, with probability $\ge 1-\ep_1$, this is bounded by
\begin{align*}
    & \lesssim 
    h\pa{\be \ka d + \be \ka \ln \prc{\ep_1}^2 + \be^2L^2 + \be^2dh} + \ka d + \ka \ln \prc{\ep_1}^2 + \be L^2 + d\ln \prc{\al h}\\
    &\lesssim
    \ka d + \ka \ln \prc{\ep_1}^2 + \be L^2 + d\ln \ka . 
\end{align*}
\end{proof}

\begin{lemma} \label{lem:better bound on init}
Suppose $\pi$ satisfies \Cref{assumption:mixture} with the Poincar\'e inequality $\PI\prc{\al}$. 
Suppose furthermore that each $\pi_i$ satisfies $\LSI(\CLS)$. 
Let $ \bar{\nu}_t$ be the distribution of the continuous Langevin diffusion wrt $\mu$ initialized at $\delta_x$, where $\max_j \ve{x-\ol x_j}< R_{\ep_1}+L$. 
Then there is $t=O\pa{\CLS \ln (\ka d\ln (
1/\ep_1) + \be L^2)}$
such that 
    \[
\ve{\dd{\bar{\nu}_{t}}{\mu}}_{L^2(\mu)}
\le \prc{\min p_i}^{1/2}e.
    \]
\end{lemma}
\begin{proof}
By choosing 
$h=\rc{50\be}$ 
and $\ep = \fc{c}{\ka d + \ka \ln (1/\ep_1)^2 + \be L^2 + d\ln 
\ka }$ for small enough constant $c$, using \Cref{c:R2-bd}, 
\begin{align*}
    \ve{\dd{\bar{\nu}_h}{\mu}}_{L^{1+\ep}(\mu)} 
    &= 
    \exp\ba{\ep \Renyi_{1+\ep}(\bar{\nu}_h ||\mu)} \le e.
\end{align*}
Now choose $t_1=\fc{\CLS}{2}\ln \prc\ep$, so $q(t_1) = 1+\ep e^{2t/\CLS}=2$. We get by hypercontractivity for mixtures (\Cref{l:hyper}) that
\begin{align*}
    \ve{\dd{\bar{\nu}_{h+t_1}}{\mu}}_{L^2(\mu)}
    &\le 
    \prc{\min p_i}^{\rc{1+\ep} - \rc 2}\ve{\dd{\bar{\nu}_h}{\mu}}_{L^{1+\ep}(\mu)} 
    \le \prc{\min p_i}^{1/2} e
\end{align*}
Finally note 
\[\ln\prc{\ep} = O\pa{\CLS \ln (\ka d + \ka \ln (1/\ep_1)^2 + \be L^2 + d\ln 
\ka)} = O\pa{\CLS \ln (\ka d\ln (
1/\ep_1) + \be L^2)}.\]
\end{proof}

\subsection{Convergence for Langevin diffusion}
\label{s:ld}

\begin{theorem}\label{l:ld}
Suppose that $\pi$ satisfies \Cref{assumption:mixture} with either the Poincar\'e inequality $\PI\prc{\al}$ or log-Sobolev inequality $\LSI\prc{\al}.$ 
Let $ \hat{\pi}$ be the uniform distribution over $n$ i.i.d. samples from $\pi$\footnote{as in \Cref{thm:main}} and $  \wt \nu_T$ be the distribution at time $t$ of the continuous Langevin diffusions initialized at the empirical distribution $\hat{\pi}$ driven by the score function $s$, where $s$ satisfies $\E_{\pi}\ve{s - \nabla \ln \pi }^2\le \esc^2$.

Suppose $n=\Om\pa{\fc{k}{\etv^2} \ln \pf k\de}$ with appropriate constants.

\begin{enumerate}
    \item 
Under $\PI\prc{\al}$, for $T\ge t_0 = 
\rc{50\be}$, with probability $\ge 1-\de$, 
    \[
\TV(\wt \nu_T, \pi) 
\lesssim 
\etv + 
e^{O\pa{\ka d + \ka \ln \pf{k}{\etv}^2 + \be L^2 + d\ln \ka 
}}e^{-\al(T-t_0)} + \sqrt T \esc + \fc{\ln (1/\de)}{n}.
\]
In particular, when $T=\Te\pa{\fc \ka\al \pa{d + \ln \pf{k}{\etv}^2} + \ka L^2 + \fc{d}{\al} \ln \ka 
+  \rc \al \ln \prc{\etv}}$, and $\esc = O\pf{\etv}{\sqrt T} $with appropriate constants, then this is $\le \etv$. 
\item 
Under $\LSI\prc{\al}$, there is 
$t_1=O\pa{\rc{\al} \ln (\ka d\ln (
k/\etv) + \be L^2)}$, 
such that for $T\ge t_1$, with probability $\ge 1-\de$, letting $p^* = \min_i p_i$, 
\[
\TV(\wt \nu_T, \pi) 
\lesssim \etv + \prc{p^*}^{1/2} e^{-\al (T-t_1)}+ \sqrt T \esc + \fc{\ln (1/\de)}{n}.
\]
In particular, when $T= t_1 + \Te \pa{\rc{\al} \ln \pf{1}{\etv  p^*} }$, and $\esc = O\pf{\etv}{\sqrt T}$ with appropriate constants, then this is $\le \etv$.
\end{enumerate}

\end{theorem}
\begin{proof}
Let $\bar{\nu}_t^{\hat{\pi}}$ be the distribution at time $t$ of the continuous Langevin diffusions initialized at the empirical distribution $\hat{\pi}$ driven by the score function $\nb \ln \pi.$
\begin{enumerate}
    \item 
    Let $R_\ep = \rc{\sqrt \al}\pa{\sqrt d + \ln \pf{3}{\ep}} $.
    Let $\ep_1 = \fc{\etv^2}{64k}$. By 
    \Cref{c:R2-bd}, 
    when $\bar \nu_0=\de_x$, with probability $\ge 1-\ep_1$ over $x\sim \pi$,
\begin{align*}
    \ln (\chi^2(\bar\nu_h\|\pi)+1) = 
    \Renyi_2(\bar \nu_h \|\pi )
    &\lesssim \ka d + \ka \ln \prc{\ep_1}^2 + \be L^2 + d\ln \ka. 
\end{align*}

By \Cref{thm:main}, with probability $\ge 1-\fc\de2$, 
\begin{align*}
    \TV(\bar{\nu}_T^{\hat{\pi}} , \pi)
    &\le \fc{\etv}{2} + e^{O\pa{\ka d + \ka \ln \prc{\ep_1}^2 + \be L^2 + d\ln \ka 
    }}e^{-\al(T-h)}.
\end{align*}
Now we consider score error. 
By \Cref{l:ld-score-error-mix} and \Cref{thm:approximate-sample-init}, with probability $\ge 1-\fc\de2$, 
\[
\TV(\bar{\nu}_T^{\hat{\pi}},\wt \nu_T)
\le 
\sqrt{2T}\esc  + \fc{2\ln (2/\de)}{n}.
\]
By the triangle inequality, with probability $\ge 1-\de$,
\[
\TV(\wt \nu_T, \pi) 
\le \fc{\etv}2 + e^{O\pa{\ka d + \ka \ln \prc{\ep_1}^2 + \be L^2 + d\ln 
\ka}}e^{-\al(T-h)} + \sqrt{2T} \esc + \fc{2\ln (2/\de)}{n}.
\]
When $T=\Te\pa{\fc \ka\al \pa{d + \ln \prc{\ep_1}^2} + \ka L^2 + \fc{d}{\al} \ln 
\ka +  \rc \al \ln \prc{\etv}}$, $\esc = O\pf{\etv}{\sqrt T}$, and $\de = e^{-O(n\etv)}$ with appropriate constants, then this is $\le \etv$. 
\item Let $\bar{\nu}_t^{x}$ be the distribution at time $t$ of the continuous Langevin diffusion initialized at $\delta_x$  driven by the score function $\nb \ln \pi.$
By \Cref{lem:better bound on init}, there is 
$t_1=O\pa{\rc\al \ln (\ka d\ln (
1/\ep_1) + \be L^2)}$ such that 
with probability $\ge 1-\ep_1$ over $x\sim \pi$, 
\[
\chi^2(\bar\nu^x_{t_1}\| \pi)\le 
    \ve{\dd{\bar \nu_{t_1}}{\pi}}_{L^2(\pi)}
    \le  \prc{\min p_i}^{1/2}e.
\]
Set $\ep_1 = O(\frac{\etv ^2}{16k})$ with appropriate constant.
By \Cref{thm:main}, with probability $\ge 1-\fc\de2$, for $T\ge t_1$,
\begin{align*}
    \TV(\bar \nu_T , \pi)
    &\le \fc{\etv}{2} + \prc{\min p_i}^{1/2}e\cdot e^{-\al(T-t_1)}.
\end{align*}
By the triangle inequality, with probability $\ge 1-\de$,
\[
\TV(\wt \nu_T, \pi) 
\le \fc{\etv}2 + \prc{\min p_i}^{1/2}e\cdot e^{-\al(T-t_1)} + \sqrt{2T} \esc + \fc{2\ln (2/\de)}{n}.
\]
When $T=t_1+\Te\pa{ \rc\al \ln \prc{\etv p^*}}$, $\esc = O\pf{\etv}{\sqrt T}$, and $\de = e^{-O(n\etv)}$ with appropriate constants, then this is $\le \etv$. 
\end{enumerate}
\end{proof}



    

\subsection{Removing dependence on the minimum weight}
\label{s:ld-wo-min}
We define some notation to refer to submixtures containing a subset of the components.
\begin{definition}\label{def:partial distribution}
Suppose that $\pi = \sumo ik p_i\pi_i$, where $p_i>0$ for each $i$ and $\sumo ik p_i=1$. 
    For set $S\subeq [k]$, let $p_S =\sum_{i\in S} p_i$ and 
    \[ \pi_S = p_S^{-1}\sum_{i\in S} p_i \pi_i. \] 
    Note that we can write $\pi(x) = p_{S} \pi_{S}(x) + p_{S^c} \pi_{S^c}(x)$
where $S^c $ is the complement of $S.$
Let $V_S$ be such that $\pi_S (x)= \exp(-V_S(x))$.
\end{definition}
\begin{lemma}\label{l:expected gradient bound}
Suppose $\pi$ satisfies \Cref{assumption:mixture} with the Poincar\'e inequality $\PI\prc{\al}$. 
    Let $G = \max_{1\le i\le k} \ve{\nb V_i(x)}$. Let $\tilde{G} = (\be \ka d)^{1/2} + \be L.$ Then for all $1\le i\le k$ and $p = O(1)$,
    \begin{align*}
    \E_{\pi_i} \ve{G(x)}^{2p} \lesssim \tilde{G}^p  \quad \text{and}\quad 
     \E_{\pi} \ve{G(x)}^{2p} \lesssim \tilde{G}^p .
\end{align*}
\end{lemma}
\begin{proof}
For all $1\le j\le k$,
    \begin{align*}
\ve{\nb V_j(x)}
&\le \ve{\nb V_j(x) - \nb V_j(\mean_j)} + \ve{\nb V_j(\mean_j)}\\
&\le \be \ve{x-\mean_j} + \ve{\nb V_j(\mean_j)}\\
&\lesssim \be \ve{x-\mean_i} + \be \ve{\mean_i - \mean_j}+\be \sfc{d}{\al} 
\end{align*}
by \Cref{l:dV-mean}.
Then by \eqref{eq:poincare moment bound},
\begin{align*}
    \ve{G(x)}^{2p} &\lesssim \be^{2p} \ve{x-\mean_i}^{2p} + \be^{2p} L^{2p} + (\be \ka d)^p\\
    \implies \E_{\pi_i} \ve{G(x)}^{2p} &\lesssim \E_{\pi_i}[\ve{x-\bar{x_i}}^{2p}] + (\be\ka d)^p + (\be L)^{2p} \lesssim (\be\ka d)^p + (\be L)^{2p}.
    \qedhere
\end{align*}
\end{proof}
\begin{lemma}[$L^2$ score error from removing small components]\label{l:score-error-small}
Suppose $\pi$ satisfies \Cref{assumption:mixture} with the Poincar\'e inequality $\PI\prc{\al}$. 
    Then 
    \[
\E_\pi \ve{\nb V(x) - \nb V_S(x)}^2 
\lesssim p_{S^c} (\be \ka d + \be^2 L^2).
    \]
\end{lemma}
\begin{proof}
First note that
    \begin{align*}
        \nabla V(x) - \nabla V_{S}(x) = \frac{\sum_{i\in S^c} p_i \pi_i(x) (\nabla V_i(x) - \nabla V_{S}(x)) }{p_S \pi_{S^c}(x) + p_{S}\pi_S(x)}.
    \end{align*}
    Letting $G(x) = \max_{1\le i\le k}\ve{\nb V_i(x)}$, note that $\ve{\nb V (x)}, \ve{\nb V_S(x)} \leq G(x)$ by \cite[Proposition 6]{koehler2023sampling}.
Thus, we have 
\begin{align}
\nonumber
    \E_\pi \ve{\nb V(x) - \nb V_S(x)}^2 
    &\le 
    4 \int \ve{G(x)}^2 \frac{p_{S^c} \pi_{S^c}(x)}{p_{S^c}\pi_{S^c}(x) + p_{S} \pi_{S}(x)} \sum_{i\in S^c} p_i \pi_i(x)\,dx\\
    &\le 
    4\sum_{i\in S^c} p_i \E_{\pi_i}\ve{G(x)}^2.
    \label{e:V-VS-error}
\end{align}
 Then \eqref{e:V-VS-error} combined with \Cref{l:expected gradient bound} give the desired bound.
\end{proof}

\begin{theorem}\label{t:ld-better}

Let $ \hat{\pi}$ be the uniform distribution over $n$ i.i.d. samples from $\pi$\footnote{as in \Cref{thm:main}} and $  \wt \nu_T$ be the distribution at time $t$ of the continuous Langevin diffusion initialized at the empirical distribution $\hat{\pi}$ driven by $s$, where $s$ satisfies $\E_{\pi}\ve{s - \nabla \ln \pi }^2\le \esc^2$.

    Suppose that $\pi$ satisfies \Cref{assumption:mixture} with the log-Sobolev inequality $\LSI\prc{\al}$, and $n=\Om\pa{\fc{k}{\etv^2} \ln \pf k\de}$ with appropriate constants.
    Let $\esm\in(0,\rc 2)$. 
    There is 
\[t_1=O\pa{\rc{\al} \ln (\ka d\ln (
k/\etv) + \be L^2)},\] 
such that for $T\ge t_1$, with probability $\ge 1-\de$, 
\[
\TV(\wt \nu_T, \pi) 
\lesssim \etv + \pf{k}{\esm}^{1/2} e^{-\al (T-t_1)}+ \sqrt T \pa{\esc + \sqrt{\esm(\be \ka d + \be^2 L^2)}} + \fc{\ln (1/\de)}{n}+ \esm.
\]

In particular, when $T=t_1+\Om\pa{\rc{\al} \ln \pa{\fc{  k  }{\etv \esm} }}$, $\esc = O\pf{\etv}{\sqrt T}$, and 
$\esm =O\pf{\etv^2}{T(\be \ka d + \be^2 L^2)}$
with appropriate constants, then this is $\le \etv$. 
The condition on $\esm$ can be satisfied by taking $T = \Te\pa{\rc \al \ln \pf{k(\ka^2 d + \be\ka L^2)}{\etv}}$ and \[\esm = O\pf{\etv^2 \al}{\ln \pf{
k\ka d + \be L^2}{\etv} (\be \ka d + \be^2 L^2)} = O\pf{\etv^2 }{\ln \pf{k \ka^2 d + \be\ka L^2}{\etv} (\ka^2 d + \be\ka L^2)} \] for appropriate constants.

\end{theorem}
\begin{proof}
    Let 
    \[
S = \bc{i:p_i>\fc{\esm}k}. 
    \]
    Then $p_{S^c} \le \esm$. 
    We have samples from $\pi$, but treat them as samples from $\pi_S$ with TV error $\le \esm$. Comparing Langevin with $\nb V$ and $\nb V_S$, by \Cref{l:score-error-small}, the score error is
    \begin{align*}
\E_{\pi_S} \ve{s(x) -(- \nb V_S(x))}^2
&\le
2\E_{\pi} \ve{s(x) - (-\nb V_S(x))}^2\\
&\lesssim \E_{\pi} \ve{s(x) -(-\nb V(x))}^2 + \E_{\pi} \ve{\nb V(x) -\nb V_S(x)}^2 \\
&\lesssim \esc^2 + p_{S^c} (\be \ka d + \be^2 L^2)
\le \esc^2 + \esm (\be \ka d + \be^2 L^2).
    \end{align*}
    Let $ \bar{\nu}_t$ be the distribution at time $t$ of the continuous Langevin diffusion initialized at the empirical distribution $\hat{\pi}$ driven by the score functions $\nabla \ln \pi_S\equiv -\nb V_S.$ 

    Note that $ \pi_S =\sum_{i\in S} p'_i \pi_i$ where $ p'_i =p_i/p_S\geq p_i > \fc{\esm}k.$
    As in the proof of \Cref{l:ld}(2), by \Cref{lem:better bound on init} and \Cref{thm:generalize main}, choosing $t_1, T$ as in the theorem statement gives with probability $\ge 1-\fc\de2$ that 
    \[
\TV(\bar \nu_T, \pi_S)\lesssim \etv + \pf{k}{\esm}^{1/2} e^{-\al(T-t_1)}.
    \]
By \Cref{l:ld-score-error-mix} and \Cref{thm:approximate-sample-init}, with probability $\ge 1-\fc\de2$, 
\[
\TV(\bar\nu_T,\wt \nu_T)
\lesssim
\sqrt{T}\pa{\esc + \sqrt{\esm (\be \ka d + \be^2 L^2)}}  + \fc{\ln (1/\de)}{n}.
\]
Noting $\TV(\pi_S, \pi)\lesssim \esm$, by the triangle inequality, with probability $\ge 1-\de$,
\[
\TV(\wt \nu_T, \pi)\lesssim 
\etv + \pf{k}{\esm}^{1/2} e^{-\al(T-t_1)} + \sqrt{T}\pa{\esc + \sqrt{\esm (\be \ka d + \be^2 L^2)}}  + \fc{\ln (1/\de)}{n} + \esm.
\]
Choosing parameters as in the theorem statement then makes this $\le \etv$.
\end{proof}

\subsection{Convergence for Langevin Monte Carlo}
\label{s:lmc}

\begin{theorem}[Langevin Monte Carlo on mixture] \label{thm:langevin discretization}
Let $\pi =\exp(-V(x))=\sum_{i=1}^k p_i \pi_i$ be a mixture of distributions  $\pi_i \propto \exp(-V_i(x))$ satisfying \cref{assumption:mixture}.
Let $\tilde{G} = (\be \ka d)^{1/2} + \be L.$ 

Let $\hat{\pi}$ be the empirical distribution i.e. the uniform distribution over be the multiset $U_{\text{sample}}$ of $y_1, \ldots, y_n$ i.i.d. sampled from $\pi.$ 

Let $s$ be such that $\E_{\pi}\ve{s - (-\nb V)}^2\le \esc^2$. 
Let $ X_t^{\hat{\pi}}$ be the LMC driven by $s$ with step size $h$ initialized at $\hat{\pi}.$
Suppose  $n \geq \Omega\left(\frac{k}{\etv ^2} \ln\pf{k}{ \de}\right).$
We state two results, corresponding to the two different functional inequalities for the mixture components. 
\begin{enumerate}
\item
Under $\PI\prc{\al},$ 
    suppose
    \[
    T =\Te\pa{\fc \ka\al \pa{d + \ln \pf{k}{\etv}^2} + \ka L^2 + \fc{d}{\al} \ln 
    \ka +  \rc \al \ln \prc{\etv}} \]
    and
    \[ h = O\left(\frac{\etv^2}{d (\tilde{G}^4 + \be^4) T }\right), \quad \esc  = O\left(\frac{\etv }{\sqrt{T}}\right), \quad \de = e^{-O(n\etv)}\]
with appropriate constants.
Then
    with probability $\geq 1-\de$ over the randomness of $U_{\text{sample}},$ 
    \[ \TV (\mathcal L(X^{\hat \pi}_T\mid \hat \pi), \pi) \le \etv . \]
    \item Under $\LSI\prc{\al}$, suppose
    \[
    T = \Te\pa{\rc \al \ln \pf{
    k\ka d + \be L^2}{\etv}},\,  h = O\pa{\frac{\etv^2}{d (\tilde{G}^4 + \be^4) T }} 
,\, \esc =  O\pa{\frac{\etv }{\sqrt{T}}},\, \de = e^{-O(n\etv)}\]
with appropriate constants.
Then
    with probability $\geq 1-\de$ over the randomness of $U_{\text{sample}},$ 
    \[ \TV (\mathcal L(X^{\hat \pi}_T\mid \hat \pi), \pi) \le \etv.  \]
\end{enumerate}
\end{theorem}
\begin{remark}
    When $L\geq 1,$ $\norm{\nabla^2 V}_{OP}$ can be as large as $\beta_* = (\beta L)^2 = \Theta (\tilde{G}^2),$ thus the step size $h =\frac{1}{d (\tilde{G}^4+\beta^4) T} =\Theta(\frac{1}{d\beta_*^2 T})$  is expected for Langevin Monte-Carlo.
\end{remark}
\begin{proof} The proof is similar to that of \Cref{l:ld}. 
We can compare between the Langevin Monte Carlo and the continuous Langevin diffusion by combining \Cref{l:ld-score-error-mix,l:expected gradient bound} and \Cref{thm:approximate-sample-init}.
Let
$ \bar{X}_t^{\hat{\pi}}$ be the continuous Langevin diffusion initialized at $\hat{\pi}$. Then with probability $\ge 1-\fc \de 2$, 
\begin{align*}
    &\TV  ( \mathcal L(\bar{X}^{\hat \pi}_T\mid \hat \pi), \mathcal L(X^{\hat \pi}_T\mid \hat \pi) ) \\
    &\lesssim 
\sqrt{T\cdot
        \pa{\be^4 (hd + h^6\tilde G^6 + h^3 d^3) + (h^2 \tilde G^2 + hd) \tilde G^4 + \esc^2}}
+ \fc{\ln (1/\de)}{n}\leq \fc{\etv}{2}\end{align*}
by our choice of $ h$ and $\esc.$

Part 1 is by combining \Cref{l:ld}(1) and the triangle inequality. Indeed, the choice of $T$ implies
\[\TV  ( \mathcal L(X^{\hat \pi}_T\mid \hat \pi), \pi) \leq \etv /2 +\TV  ( \mathcal L(\bar{X}^{\hat \pi}_T\mid \hat \pi), \pi) \leq \etv \]
Similarly, part 2 follows by combining \Cref{t:ld-better} and the triangle inequality.

\end{proof}

\section{Application to low-complexity Gibbs measures}\label{s:learning-ising}
\subsection{Multiplicative mixture approximation}
The following result shows that for spin systems with pairwise/quadratic interactions, we can eliminate large positive eigenvalues in the interaction matrix at the cost of decomposing them as ``small'' mixture distributions. The main idea is to apply the Hubbard-Stratonovich transform to the large eigendirections of the interaction matrix, which naturally decomposes the distribution into an explicit infinite mixture, and then argue that this decomposition can be discretized to a finite mixture with small loss. This is inspired by \cite{koehler2022sampling}, where a related idea was used to build an approximate sampling algorithm in the case of the Ising model (i.e. $\mu_i$ supported on $\{\pm 1\}$). The argument here is done in a more general setting and crucially guarantees a multiplicative approximation to the true density. 
\begin{theorem}\label{thm:mult-approx-mixture}
Suppose that $\mu = \mu_1 \otimes \cdots \otimes \mu_n$ is a product measure where each $\mu_i$ is supported within a Euclidean ball in $\mathbb{R}^d$ of radius $D$. Suppose that the measure $\pi$ is absolutely continuous to $\mu$ with Radon-Nikodym derivative
\[ \frac{d\pi}{d\mu}(x) \propto \exp\left(\frac{1}{2} \langle x, J x \rangle\right). \]
Suppose that $J = J_+ + \tilde J$ and that $J_+$  has rank $r$ with eigenvalues $\lambda_1 \ge \lambda_2 \ge \cdots \ge \lambda_r > 0$. Then there exists a set $S$ and a mixture distribution $\pi_2$ with mixing weights $p_h > 0$ for $h \in S$ such that:
 \[ \pi_2(x) = \sum_{h \in S} p_h \tilde{\pi}_h(x) \]
 where 
 $\tilde{\pi}_h$ is the probability measure with normalizing constant $Z_H$ satisfying
\[ \frac{d\tilde{\pi}_h}{d\mu}(x) = \frac{1}{Z_h} \exp\left(\langle x, \tilde J x \rangle + \langle h, x \rangle\right). \]
Furthermore, 
\[ \frac{d\pi_2}{d\pi}(x) \in [1/e^3,e^3] \]
for all $x$, and
\[ |S| \le \left(D^2 n \lambda_1 + D\sqrt{n}r\sqrt{\lambda_1} + \sqrt{\lambda_1 D^2 n r \ln(\lambda_r^{-1/2} + D\sqrt{n})}\right)^{O(r)}.  \]
\end{theorem}
\begin{remark}
In the greater than one-dimensional case, we view $x$ as an $n \times d$ matrix and the inner product is the usual Frobenius inner product on matrices.
\end{remark}
\begin{remark}
    The proof can be straightforwardly modified to replace the interval $[1/e^3,e^3]$ by one of width $[1 - \delta, 1 + \delta]$ for any $\delta > 0$, with a polylogarithmic dependence on $\delta$. 
\end{remark}
\begin{remark}[Exponential dependence on $r$ is necessary]
    Consider an Ising model which is a disjoint union of $r$ low temperature Curie-Weiss models each on $m$ spins. Since each Curie-Weiss model individually has a bimodal limiting distribution for the average magnetization $\frac{1}{m}\sum_i X_i$ (see e.g. \cite{ellis2007entropy}) as $m \to \infty$, the joint distribution of the $r$ magnetizations will have $2^r$ well-separated modes.
\end{remark}
\begin{proof}
Using the moment generating function of a Gaussian distribution (Hubbard-Stratonovich identity) we have
\[ e^{\|J_+^{1/2} x\|_2^2/2} = \E_{H \sim N(0, J_+)}\left[e^{\langle x, H \rangle}\right]. \] 
Therefore
\begin{align*} 
\frac{d\pi}{d\mu}(x) 
&\propto \E_{H \sim N(0,J_+)}\left[ \exp\left(\frac{1}{2} \langle x, \tilde J x \rangle + \langle x, H \rangle\right)\right]  \\
&\propto \E_{H \sim N(0,J_+)}\left[ \frac{Z_H}{Z_H}\exp\left(\frac{1}{2} \langle x, \tilde J x \rangle + \langle x, H \rangle\right)\right]  \\
&\propto \E_{H \sim N(0,J_+)}\left[ Z_H \tilde{\pi}_H(x) \right] \\
&\propto \int_{\im(J_+)} \tilde{\pi}_H(x) Z_H e^{-\langle H, J^{\dagger}_+ H \rangle/2} dH
\end{align*}
where the last integral is with respect to Lebesgue measure on the hyperplane $\im(J_+)$, and where $J_+^{\dagger}$ is the pseudoinverse of $J_+$. 
\[ \frac{d\tilde{\pi}_H}{d\mu}(x) = \frac{1}{Z_H} \exp\left(\langle x, \tilde J x \rangle + \langle H, x \rangle\right). \]

Let $R > 0$ to be fixed later and let $B_R(0)$ be the ball of radius $R$ centered at the origin. We have that
\[ \frac{d\pi}{d\mu}(x) \propto \int_{\im(J_+) \cap B_R(0)} \tilde{\pi}_H(x) Z_H e^{-\langle H, J^{\dagger}_+ H \rangle/2} dH +  \int_{\im(J_+) \setminus B_R(0)} \tilde{\pi}_H(x) Z_H e^{-\langle H, J^{\dagger}_+ H \rangle/2} dH\]
and for any $H$, we know by the Cauchy-Schwarz inequality that
\begin{equation}\label{eqn:part-ratio}
Z_H/Z_0 = \mathbb{E}_{X \sim \tilde \pi_0} e^{\langle H, X \rangle} \in [e^{-\|H\|_2 D \sqrt{n}},e^{\|H\|_2 D \sqrt{n}}], 
\end{equation}
and similarly
\begin{equation}\label{eqn:dens-ratio}
\tilde{\pi}_H(x)/\tilde{\pi}_0(x) = \frac{e^{\langle H, X \rangle}}{Z_H} \in [e^{-2\|H\|_2 D \sqrt{n}},e^{2\|H\|_2 D \sqrt{n}}]. 
\end{equation}
Let $C_r$ be the surface area of the ($r-1$ dimensional) unit sphere in $r$ dimensions. 
It follows that if $J_+$ has rank $r$ with eigenvalues $\lambda_1 \ge \lambda_2 \ge \lambda_r > 0$ that
\begin{align*} 
\int_{\im(J_+) \setminus B_R(0)} \tilde{\pi}_H(x) Z_H e^{-\langle H, J^{\dagger}_+ H \rangle/2} dH 
&\le \tilde{\pi}_0(x) Z_0 \int_{\im(J_+) \setminus B_R(0)}  e^{3\|H\|_2 D \sqrt{n} -\lambda_1^{-1} \|H\|_2^2/2} dH \\
&\le C_r \tilde{\pi}_0(x) Z_0 \int_{R}^{\infty} s^{r - 1}  e^{3s D \sqrt{n} -\lambda_1^{-1} s^2/2} ds \\
&\le C_r \tilde{\pi}_0(x) Z_0 e^{-\lambda_1^{-1} R^2/4} \int_{R}^{\infty} s^{r - 1}  e^{3s D \sqrt{n} -\lambda_1^{-1} s^2/4} ds \\
&\le 2 C_r \tilde{\pi}_0(x) Z_0 e^{-\lambda_1^{-1} R^2/4}
\end{align*}
where in the first inequality we applied \cref{eqn:part-ratio,eqn:dens-ratio},
in the second inequality we rewrote the integral in terms of spherical coordinates, and
the last step holds under the requirement $R = \Omega(D\sqrt{n}\lambda_r + r \lambda_1)$.

By similar arguments, we have that
\begin{align*} 
\int_{\im(J_+) \cap B_R(0)} \tilde{\pi}_H(x) Z_H e^{-\langle H, J^{\dagger}_+ H \rangle/2} dH 
&\ge \tilde{\pi}_0(x) Z_0 \int_{\im(J_+) \cap B_R(0)}  e^{-3\|H\|_2D\sqrt{n} - \lambda_r^{-1}\|H\|_2^2/2} dH \\
&\ge C_r \tilde{\pi}_0(x) Z_0 \int_{0}^R s^{r - 1}  e^{-3 s D\sqrt{n} - \lambda_r^{-1} s^2/2} ds \\
&\ge C_r \tilde{\pi}_0(x) Z_0 \int_{0}^{1/(\lambda_r^{-1/2} + D\sqrt{n})} s^{r - 1}  e^{-3 s D\sqrt{n} - \lambda_r^{-1} s^2/2} ds \\ 
&\ge e^{-4} (C_r/r) \tilde{\pi}_0(x) Z_0 [\lambda_r^{-1/2} + D\sqrt{n}]^{-r}.
\end{align*}

Hence
\[ \frac{\int_{\im(J_+) \setminus B_R(0)} \tilde{\pi}_H(x) Z_H e^{-\langle H, J_+ H \rangle/2} dH }{\int_{\im(J_+) \cap B_R(0)} \tilde{\pi}_H(x) Z_H e^{-\langle H, J_+ H \rangle/2} dH } \le 2e^4 r [\lambda_r^{-1/2} + D\sqrt{n}]^r e^{-\lambda_1^{-1} R^2/4} < 1/4  \]
provided that we additionally require $R = \Omega\left(\sqrt{r\lambda_1 \ln(\lambda_r^{-1/2} + D\sqrt{n})}\right)$.

Therefore if we define $\pi_1$ to be the probability measure with density
\[ \frac{d\pi_1}{d\mu}(x) \propto \int_{\im(J_+) \cap B_R(0)} \tilde{\pi}_H(x) Z_H e^{-\langle H, J_+ H \rangle/2} dH \]
we have shown that $\frac{d\pi_1}{d\pi}(x) \in [3/4, 5/4]$ for all $x$. Let $\delta > 0$ and $N_{\delta}(R)$ to be an $\delta$-net of $\im(J_+) \cap B_R(0)$. For any point $h \in N_{\delta}(R)$, define the set $S_h$ to be the subset of points in $\im(J_+) \cap B_R(0)$ such that $h$ is the closest point in the net, i.e. $S_h$ is the Voronoi region of $h$. Define probability measure
\[ \frac{d\pi_2}{d\mu}(x) \propto \sum_{h \in N_{\delta}(R)} \tilde{\pi}_h(x) \int_{S_h}  Z_H e^{-\langle H, J_+ H \rangle/2} dH. \]
Since $\tilde{\pi}_h(x)/\tilde{\pi}_H(x) \in [e^{-2 \delta D \sqrt{n}}, e^{2 \delta D \sqrt{n}}]$, it follows that if $\delta = 1/2D \sqrt{n}$ then for any $H \in N_{\delta}(R)$,
$\tilde{\pi}_h(x)/\tilde{\pi}_H(x) \in [1/e, e]$. Using that
\[ \frac{d\pi_2}{d\pi_1}(x) \propto \frac{\sum_{h \in N_{\delta}(R)} \tilde{\pi}_h(x) \int_{S_h}  Z_H e^{-\langle H, J_+ H \rangle/2} dH}{\int_{\im(J_+) \cap B_R(0)} \tilde{\pi}_H(x) Z_H e^{-\langle H, J_+ H \rangle/2} dH} \]
we conclude that $\frac{d\pi_2}{d\pi_1}(x)\in [1/e^2,e^2]$ for all $x$, hence $\frac{d\pi_2}{d\pi}(x) \in [1/e^3,e^3]$ for all $x$.

Finally, by standard covering number bounds (see e.g. \cite{vershynin2018high}) we observe that $\pi_2$ is a mixture of
\[ |N_{\delta}(R)| \le (R/\delta)^{O(r)} \]
many components. Taking
\[ R = \Theta\left(\lambda_1 D\sqrt{n} + r\sqrt{\lambda_1} + \sqrt{\lambda_1 r \ln(\lambda_r^{-1/2} + D\sqrt{n})}\right) \]
yields that
\[ R/\delta = D\sqrt{n} \cdot \Theta\left(D\sqrt{n}\lambda_1 + r\sqrt{\lambda_1} + \sqrt{\lambda_1 r \ln(\lambda_r^{-1/2} + D\sqrt{n})}\right) \]
as claimed. 
\end{proof}
\subsection{Exact mixture decomposition and higher-order spectral gap}
The previous section gave an approximate mixture decomposition, but when the mixture components satisfy a functional inequality, we now show that it leads to an exact mixture decomposition of components satisfying the same functional inequality. This in particular yields a higher-order spectral gap estimate. 

We now give an example to illustrate the idea. For concreteness, we state the following result for Ising models, but as discussed after the proof it can be adapted to other similar models like the Potts model or $O(N)$ model, since \cref{thm:mult-approx-mixture} is a general result. For example, we illustrate the application of the result to prove a higher-order spectral gap for Glauber dynamics in the mean-field Potts model.
\begin{theorem}\label{thm:ising-higher-order-gap}
Let $c \in [1, \infty)$.
Suppose that $\pi$ is a probability measure on $\{ \pm 1\}^n$ satisfying
\[ \pi(x) \propto \exp\left((1/2)\langle x, J x \rangle + \langle h, x \rangle\right) \]
for some symmetric interaction matrix $J$ and external field vector $h$.
Suppose that $J$ has $r$ eigenvalues greater than $ 1-1/c$ and the rest are at most $1-1/c$, i.e. the eigenvalues of $J$ are $\lambda_1 \ge \lambda_2 \ge \cdots \ge \lambda_r > 1- 1/c \geq \lambda_{r+1} \geq \cdots \geq \lambda_n.$ 
Then $\pi$ admits a mixture decomposition
\[ \pi(x) = \sum_{h \in S} q_h \overline{\pi}_h(x) \]
where $q_h \ge 0, \sum_h q_h = 1$, 
\[ |S| \le \left(n \lambda_1 + \sqrt{n}r\sqrt{\lambda_1} + \sqrt{\lambda_1 n r \ln(\sqrt{\lambda_r^{-1/2}} + \sqrt{n})}\right)^{O(r)},  \]
and each component $\overline{\pi}_h$ satisfies the Poincar\'e inequality for the continuous-time Glauber dynamics with constant
\[ \exp\left[O\left(\int_0^{T_0} c e^{c s \Tr(J_-)} ds\right) \right] \]
where $-J_-$ is the negative definite part of $J$ and $T_0  = 1-1/c - \lambda_{\min} (J)$. 
Furthermore, $\pi$ satisfies the higher-order spectral gap inequality
\[ \lambda_{|S| + 1}\left(-\sL\right) \ge \exp\left[-O\left(\int_0^{T_0} c e^{c s \Tr(J_-)} ds\right)\right]\]
where $\sL$ is the generator of the continuous-time Glauber dynamics for $\pi$.
\end{theorem}
\begin{proof}
By \cref{thm:mult-approx-mixture}, there exists a distribution $\pi_2$ such that
\[ \pi_2(x) = \sum_{h \in S} p_h \tilde{\pi}_h(x) \]
 where 
 $\tilde{\pi}_H$ is the probability measure satisfying
\[ \tilde{\pi}(x) \propto \exp\left(\langle x, \tilde J x \rangle + \langle H, x \rangle\right) \]
and 
$\tilde{J} = \sum_{i=r+1}^{n} \lambda_{i}\phi_i \phi_i^\intercal$ 
where $\lambda_1\geq \cdots \geq \lambda_r > 1-1/c\ge  \lambda_{r+1} \ge\cdots \ge \lambda_n$ are the eigenvalues of $J$ and $\phi_i$ are the eigenvector corresponds to $ \lambda_i.$

Furthermore, 
\[ d\pi_2/d\pi(x) \in [1/e^3,e^3] \]
for all $x$, and
\[ |S| \le \left(n\lambda_1 + \sqrt{n}r\sqrt{\lambda_r} + \sqrt{\lambda_1 n r \ln(\lambda_r^{-1/2} + \sqrt{n})}\right)^{O(r)}.  \]
It follows that
\[ \pi(x) = \frac{\pi(x)}{\pi_2(x)} \pi_2(x) = \sum_{h \in S} p_h \frac{\pi(x)}{\pi_2(x)} \tilde{\pi}_h(x) =  \sum_{h \in S} p_h \left(\sum_x \frac{\pi(x)}{\pi_2(x)} \tilde{\pi}_h(x)\right) \frac{\frac{\pi(x)}{\pi_2(x)} \tilde{\pi}_h(x)}{\sum_x  \frac{\pi(x)}{\pi_2(x)} \tilde{\pi}_h(x)}.  \]
By Theorem 103 of \cite[Appendix B]{anari2024trickle}, the measure $\tilde{\pi}_h$ satisfies the Poincar\'e inequality for the continuous-time Glauber dynamics with constant
\[ \exp\left(\int_0^{T_0} c e^{c s \Tr(J_-)} ds \right)  \]
where $-J_-$ is the negative definite part of $\tilde J$ (so $J_-$ is positive definite), which is the same as the negative definite part of $J$ and $T_0 = \lambda_{\max}(\tilde{J}) - \lambda_{\min}(\tilde{J})  \leq 1-1/c - \lambda_{\min} (J)$.
The probability density
\[ \overline{\pi}_h =  \frac{\frac{\pi(x)}{\pi_2(x)}\tilde{\pi}_h(x)}{\sum_x  \frac{\pi(x)}{\pi_2(x)} \tilde{\pi}_h(x)}\]
has a density with respect to $\tilde \pi$ which is upper and lower bounded by a constant (since this is true for $\frac{\pi(x)}{\pi_2(x)}$, and using this we know the denominator is also upper and lower bounded by a constant), so by Proposition 4.2.7 of \cite{bakry2014analysis} we see that $\overline{\pi}_h$ satisfies the Poincar\'e inequality as stated in the theorem. 

Finally, the higher-order spectral gap inequality follows by applying \cref{l:egap}. 
\end{proof}
\begin{remark}[Generalizations such as Potts model]
It is straightforward to extend the above result to similar models such as the Potts model under the same conditions on $J$. 
In the following theorem, for illustration, we prove an analogous result for the well-studied mean-field Potts model on $q$ coclors.
When $q = 2$, this is the same as the Curie-Weiss model, which is discussed more in \cref{s:curie-weiss}.
When $q > 2$, the model has a qualitatively different behavior from the Curie-Weiss model which is called a ``first-order phase transition''. The implication for the Glauber dynamics is that in certain regimes, the dynamics with worst-case initialization can become trapped in a metastable region which has negligible mass under the true Gibbs measure; as our techniques rigorously show, this is not a problem for data-based initialization. See e.g.\ \cite{cuff2012glauber} for extensive discussion of this model.  
\end{remark}
\begin{theorem}[Application to Mean-field Potts Model]\label{thm:mf-potts}
Suppose that $q \ge 2$ and $\pi$ is the distribution on $[q]^n$ with probability mass function
\[ \pi(x) \propto \exp\left(\frac{\beta}{2n} \sum_{i,j} 1(x_i = x_j)\right). \]
Then $\pi$ admits a mixture decomposition
\[ \pi(x) = \sum_{h \in S} q_h \overline{\pi}_h(x) \]
where $q_h \ge 0, \sum_h q_h = 1$, 
\[ |S| \le \left(n\beta + \sqrt{n}r\sqrt{\beta} + \sqrt{\beta n r \ln(n)}\right)^{O(r)},  \]
and each component $\overline{\pi}_h$ satisfies the Poincar\'e inequality for the continuous-time Glauber dynamics with constant $O(1)$. 
Furthermore, $\pi$ satisfies the higher-order spectral gap inequality
\[ \lambda_{|S| + 1}\left(-\sL\right) = \Omega(1) \]
where $\sL$ is the generator of the continuous-time Glauber dynamics for $\pi$.
\end{theorem}
\begin{proof}
If $\beta < 1/2$ then the distribution satisfies the classical Dobrushin's condition (see e.g.\ path coupling discussion in \cite{levin2017markov}) so the conclusion is trivial. Otherwise, we can assume $\beta \ge 1/2$.

Note that we can rewrite
\[ 1(x_i = x_j) = \langle e_i, e_j \rangle \]
so we can view $\pi$ as a distribution with spins in $q$ dimensions.
Therefore, we can apply \cref{thm:mult-approx-mixture} to get a multiplicative approximation to the distribution as a mixture of product measures, and then by applying the change of measure argument from the proof of \cref{thm:ising-higher-order-gap} proves the result.
\end{proof}
\subsection{Fast mixing from data-based initialization \& learning}
As an application of our techniques, we now show how our ideas lead to provably learn a natural class of Ising models which were not covered by previous results. The algorithm is the natural combination of pseudolikelihood estimation with a warm-started Gibbs sampler, and hence is not much harder to implement than pseudolikelihood itself.  

\paragraph{Pseudolikelihood estimator.} It is helpful to first recall the definition of the pseudolikelihood estimator \cite{besag1975statistical} in the case of the Ising model. Recall that in the Ising model $\pi(x) \propto \exp(\langle x, J x \rangle/2 + \langle h, x \rangle)$,
\[ \pi(X_i = x_i \mid X_{\sim i}) \propto \exp(J_{i, \sim i} \cdot X_{\sim i} x_i + h_i x_i) \]
and taking into account the normalizing constant, we have
\[  \pi(X_i = x_i \mid X_{\sim i}) = \frac{1}{1 + \exp(-2 J_{i, \sim i} \cdot X_{\sim i} x_i -2 h_i x_i) }. \]
Therefore
\[ \ln  \pi(X_i = x_i \mid X_{\sim i}) = -\ln\left(1 + \exp(-2 J_{i, \sim i} \cdot X_{\sim i} x_i -2 h_i x_i)\right) \]
and this is what is known as a logistic regression model. (See e.g. \cite{mccullagh2019generalized}.) The pseudolikelihood of an outcome $x \in \{\pm 1\}^n$ under the model $\pi$ is the sum of these conditional likelihooods over the choice of index $i$, i.e.
\[ \sum_{i = 1}^n \ln  \pi(X_i = x_i \mid X_{\sim i} = x_{\sim i}) \]
and the pseudolikelihood estimator is given by optimizing this objective averaged over a dataset.

\begin{theorem}
\label{t:learn-ising}
Fix $ c\in [2,\infty).$
    Suppose that $\pi$ is a probability measure on $\{ \pm 1\}^n$ satisfying
\[ \pi(x) \propto \exp\left((1/2)\langle x, J x \rangle + \langle h, x \rangle\right) \]
for some symmetric interaction matrix $J$ and external field vector $h$.
Suppose that $J$ has $r$ eigenvalues greater than $ 1-1/c$, including its top eigenvalue $\lambda_1$, and the rest are at most $1-1/c$. 

Suppose $X^{(0)}, \ldots, X^{(m)} \sim \pi$ are i.i.d. samples; let $\hat \E$ denote the corresponding empirical expectation over the empirical measure $\hat \pi$, i.e. $\hat \E$ is the average over the sample set. Let $Y^{(0)}, \ldots, Y^{(m)} \sim \pi$ be another set of iid samples at denote their empirical measure by $\hat \pi_2$.
Furthermore, suppose that $R \ge 0$ is such that
\begin{equation} \max_i \sum_j |J_{ij}| + |h_i| \le R. \label{eqn:constraint}
\end{equation}
Define the constrained pseudolikelihood estimator by
\[ \hat \rho = \arg\max_{\rho \in \mathcal P_R} \sum_{i = 1}^n  \hat \E_{X} \ln \rho(X \mid X_{\sim i})\]
where $\mathcal P_R$ is a set consisting of Ising models (i.e. measures on $\{ \pm 1\}^n$ with quadratic log-likelihood) with parameters satisfying the convex constraint \eqref{eqn:constraint} (so in particular, it can be optimized in polynomial time). 

With probability at least $1 - \delta$ over the randomness of the sample set, for any 
\[ t \ge T = \exp\left[O\left(\int_0^{T_0} c e^{c s \Tr(J_-)} ds \right)\right]\left(2Rn + \ln \pf{4 k \ln (k/
\delta)}{m} \right) \]
where $T_0  = 1-1/c - \lambda_{\min}(J)$
we have
\begin{align*} 
&\TV (\cal L(X^{\hat \pi_2}_t \mid \hat \pi, \hat \pi_2), \pi) \le \sqrt{tRn \sqrt{\ln(2n/\delta)/m}} + 4\sqrt{k \ln(k/\delta)/m}.
\end{align*}
where
\[k = O\left(\lambda_1 n + \sqrt{\lambda_1 n} r +  \sqrt{\lambda_1 n r \ln( \sqrt{n})}\right)^{O(r)} = \tilde{O}(\lambda_1 n)^{O(r)} \]
and $X^{\hat \pi_2}_t$ is the output of the $\hat \rho$-Glauber diffusion initialized at $\hat \pi_2$ and run for time $t$.
\end{theorem}
After the proof of the result, we give a number of remarks to elaborate on its meaning and significance.

\begin{proof}[Proof of \cref{t:learn-ising}]
By the standard symmetrization and bounded differences concentration argument from statistical learning theory (see e.g.\ \cite{bartlett2002rademacher}), we know that with probability at least $1 - \delta$ 
\begin{equation}\label{eqn:uniform-convergence}
\sum_{i = 1}^n \E_{X \sim \pi} \ln \hat \rho(X \mid X_{\sim i}) \ge \sum_{i = 1}^n \E_{X \sim \pi} \ln \pi(X \mid X_{\sim i}) - O\left(\mathcal R_m + Rn\sqrt{\ln(2/\delta)/m}\right)  
\end{equation}
where
\[ \mathcal R_m = \mathbb E_{X^{(0)}, \ldots, X^{(m)}, \epsilon} \sup_{J,h} \frac{1}{m n} \sum_{a = 1}^m \epsilon_a \sum_{i = 1}^n \ell\left(J_{i, \sim i} \cdot X^{(a)}_{\sim i} X^{(a)}_i + h_i X^{(a)}_i\right) \]
is the Rademacher complexity of the loss class
and $\ell(z) = \ln(1 + e^{2z})$ is the logistic loss, and we used that the loss in any example is bounded by $O(R)$ because the logistic loss is lipschitz and because of Holder's inequality. Applying Talagrand's contraction principle and Holder's inequality yields that
\[ \mathcal R_m = O(Rn \sqrt{\ln(n)/m}), \]
see e.g.\ the proof of Theorem 47 in \cite{anari2024universality} for the details. 

Using the Rademacher complexity bound and rearranging \eqref{eqn:uniform-convergence} yields
\[ \sum_{i = 1}^n \E \KL(\pi(\cdot \mid X_{\sim i}), \hat \rho(\cdot \mid X_{\sim i})) = \sum_{i = 1}^n \E_{X \sim \pi} \ln \frac{\pi(X \mid X_{\sim i})}{\hat \rho(X \mid X_{\sim i})} = O\left(Rn \sqrt{\ln(2n/\delta)/m}\right) \]
and a bound on the left hand side is exactly what we need to apply \cref{lem:gibbs-approximate}, which in turn lets us apply \cref{thm:approximate-sample-init} with $\varepsilon = O\left(Rn \sqrt{\ln(2n/\delta)/m}\right)$. The result is that for the empirical distribution $\hat \pi_2$ formed from $m$ samples, with probability at least $1 - \delta$,
\[ \TV (\mathcal L((X^{\hat \pi_2}_t)_{0\le t\le T}\mid \hat \pi, \hat \pi_2), \mathcal L((\tilde X^{\hat \pi_2}_t)_{0\le t\le T} \mid \hat \pi, \hat \pi_2)) \le \sqrt{T\ep} + \fc{2\ln(1/\delta)}{m} \] 
where $(X_t^{\hat \pi_2})_t$ is the law of the $\hat \rho$-Glauber dynamics initialized from $\hat \pi_2$, and $\tilde X^{\hat \pi_2}_t$ is the law of the $\pi$-Glauber dynamics with the same initialization. Note that by the union bound, the two events we are using hold together with probability at least $1 - 2\delta$. 

Recall from \cref{thm:ising-higher-order-gap} that there exists 
\[ k = O\left(n\lambda_1 + \sqrt{n}r\sqrt{\lambda_1} + \sqrt{\lambda_1 n r \ln(\lambda_r^{-1/2} + \sqrt{n})}\right)^{O(r)} \]
such that
\[ \lambda_k(- \sL) \ge \exp\left[-O\left(\int_0^{T_0} c e^{c \Tr(J_-)} ds\right)\right] \]
where $\sL$ is the generator of the Glauber dynamics.

It then follows our main result, \cref{thm:main}, applied with $t_0 = 0$ we have that with probability at least $1 - \delta$ over the randomness of the sample,
\[ \TV (\cal L(\tilde X^{\pi}_T \mid \hat \pi, \hat \pi_2), \pi) \le \varepsilon_2 \]
provided $t \ge T = (1/\lambda_k(-\sL)) \ln\left(\frac{4 e^{2Rn}}{\varepsilon_2^2} \right)$ and $m = \Omega(\varepsilon_2^{-2} k \ln(k/\delta))$ where $e^{2Rn}$ is an upper bound on the initial $\chi^2$-distance between $\hat \pi_2$ and $\pi$. We take $\ep_2 = \sqrt{k \ln(k/\delta)/m}$.

Thus, by the union bound, the data processing inequality, and the triangle inequality for total variation distance, we have with probability at least $1 - 3\delta$ that
\begin{align*} 
&\TV (\cal L((X^{\hat \pi_2}_t \mid \hat \pi, \hat \pi_2), \pi) \\
&\le\sqrt{t\ep} + \fc{2\ln(1/\delta)}{m} + \ep_2  \\
&= \sqrt{tRn \sqrt{\ln(2n/\delta)/m}} + \fc{2\ln(1/\delta)}{m} + \sqrt{k \ln(k/\delta)/m}.
\end{align*}
The bound is trivial unless $m = \Omega(k \ln(k/\delta))$, so the middle term can be dominated by the right one by adjusting the multiplicative factor in front. 
\end{proof}

\begin{remark}[Separation from previous results]\label{remark:separation}
For example, consider the case where $J$ is positive semidefinite and rank $r=O(1).$  Then for 
\[t = O( Rn + \ln(4/\etv ^2)) \]
and 
\[ m =\Omega(\max\set{ \etv ^{-2} k\ln(k/\delta),\etv ^{-4} (tRn)^2 \ln(2n/\delta) }) ,\]  we are guaranteed that
\[\TV (\cal L((X^{\hat \pi}_t \mid \hat \pi), \pi) \leq \etv  .\]
Note that both $k$ and $ m $, the number of samples, are polynomial in all parameters as long as $r = O(1).$
In contrast, previous bounds for learning Ising models with computationally efficient algorithms have exponential dependence on $R$ (see e.g. \cite{lokhov2018optimal,gaitonde2024unified,klivans2017learning,wu2019sparse}), and in some simple examples of such Ising models (e.g. in Hopfield networks with finitely many memories \cite{talagrand2010mean}) $R$ is polynomial in $n$. $R$ can also be large simply if the external field of the model is large. Information-theoretically, such models can be learned with polynomial sample complexity without any dependence on $R$ \cite{devroye2020minimax}.

A key reason why previous analyses fail is that these works all seek to first show that the parameters are identifiable, and then argue that recovering the ground truth parameters also implies recovery in TV distance. However, there are models in our class which have very tiny TV distance despite having well-separated parameters. For example, consider the class of $n + 1$ ferromagnetic Ising models with uniform edge weight $R/n$ on a graph $G$, where $G$ is either a clique or a clique with an unknown edge removed. Following the proof of Theorem 1 in \cite{santhanam2012information}, we know that the KL diveregence between any two models in this class is  $e^{-\Omega(R)}$ (Lemma 2 of \cite{santhanam2012information}), so by Fano's lemma $e^{\Omega(R)} \ln(n)$ samples from the model are needed to identify what the underlying graph $G$ is. So for example, if we take $R = n$, we will need $e^{\Omega(n)}$ many samples to determine the underlying graph, or to determine the parameters of the model within $\ell_1$ error $o(1)$. Nevertheless, our theorem shows that it is possible to learn the model to $o(1)$ error in TV from only $poly(n)$ samples.
\end{remark}
\begin{remark}
Continuing the discussion from the previous remark,
we do not currently know if pseudolikelihood estimation alone can possibly achieve a similar guarantee. In general, the behavior of pseudolikelihood estimation is closely related to restricted versions of the log-Sobolev and Poincar\'e constants \cite{koehler2022statistical}, where the restriction is to functions related to the relative density between models in the class of distributions being fit.  

Based on an analogy to the lower bound results of \cite{koehler2022statistical} for score matching, we expect that in variants of this setting where we introduce a small additional term in the log-likelihood that this approach will still work (this can be proved by a perturbation argument) but pseudolikelihood will fail. 
\end{remark}
\begin{remark}
    \Cref{t:learn-ising} also implies a fast-amortized way to produce many samples from the low-rank Ising model. For example, we can use \cite{koehler2022sampling}'s algorithm to produce  $y_1,\dots, y_m$ i.i.d. sampled from a distribution $\nu$ with $ \TV (\nu, 
    \pi)\leq \etv /16$, then run the Glauber dynamics initialized at a randomly chosen sample $y_i$ with $i \sim \textup{Uniform}([m]).$ We set $\delta=\etv /6$. By \Cref{thm:generalize main} and triangle inequality, the output of Glauber is $\etv $-close to $\pi$ in $\TV$-distance. The amortized runtime (say, in the setting of \Cref{remark:separation}) is $ O(Rn + \ln(\rc {\etv })) $ since we can ignore the cost of producing the initial $m$ samples i.e. the runtime cost to produce $M$ samples is
    \[T_{\textup{init}} + O\pa{M Rn + \ln\prc {\etv }}.\]
\end{remark}

\section{Examples with non-sample initialization}\label{s:extra-examples}
We have shown that MCMC approximately samples the target distribution in a polynomial time, as long as it is started from an initialization which approximately satisfies the eigenfunction balance condition. Data-based initialization is one way to find such an initialization, but it is not the only way. We discuss this more in this section. 

\subsection{Example: Curie-Weiss}\label{s:curie-weiss}
In this section, we show that the spectral properties of the Glauber dynamics in the low-temperature Curie-Weiss model ensure that a wide range of weakly balanced initializations mix to stationarity in polynomial time. This model is a good pedagogical example to consider, in particular because it has received a lot of previous in the rigorous literature on metastability of the Glauber dynamics. 

As a reminder, it is known that the spectral gap of the Glauber dynamics is exponentially small in $n$ for any fixed inverse temperature $\beta > 1$, while the mixing is polynomial time for $\beta \le 1$. Thus, the ``low temperature'' regime $\beta > 1$ is the one of interest for metastability. Informally, the bottleneck in the dynamics after the phase transition occurs because typical samples from the measure have either significant positive or negative magnetization, whereas the measure has very little support in the middle near zero magnetization.
Our general result \cref{thm:ising-higher-order-gap} applies to this model, and guarantees that there are $O_{\beta}(1)$ many exponentially small eigenvalues --- now, building on results from the literature, we will show there is in fact only one bad nontrivial eigenvalue.  

\textcite{levin2010glauber} (see also \cite{ding2009censored}) proved via a coupling argument that a ``censored'' version of the Glauber dynamics for the low-temperature Curie-Weiss model mixes in $O_{\beta}(n \log n)$ steps; this dynamics is restricted to spins with positive magnetization. As a warmup, we observe this implies a restricted version of the Poincar\'e inequality for even functions:
\begin{lemma}\label{lem:cw-even}
Suppose $n$ is odd, $\beta > 1$, and consider the Curie-Weiss model on $\{\pm 1\}^n$
\[ p(x) \propto \exp\left(\frac{\beta}{2n}\left(\sum_i x_i\right)^2\right). \]
Let $f$ be an even function in the sense that $f(x) = f(-x)$ for all $x \in \{\pm 1\}^n$. Then
\[ \Var(f) \le C(\beta) \mathcal E(f,f) \]
where $C(\beta) > 0$ is a constant independent of $n$.
\end{lemma}
\begin{proof}
We show that this follows from the result of \cite{levin2010glauber}. First, we recall that they proved $O_{\beta}(n \log n)$ time mixing of the following ``censored'' dynamics on the restriction of $p$ to $\{ x : \sum_i x_i > 0 \}$, which has the following behavior in one step: 
\begin{enumerate}
    \item Given $x \sim \{\pm 1\}^n$, generate a proposal $y$ via the standard Glauber dynamics.
    \item If $\sum_i y_i > 0$, transition to $y$. If $\sum_i y_i < 0$, transition to $-y$.
\end{enumerate}
By the well-known connection between mixing time and spectral gap, the mixing time result of \cite{levin2010glauber} in particular implies that the spectral gap of the  generator of the continuous-time censored dynamics is $\Omega_{\beta}(1)$. Next, observe that if $f$ is even, then $\Var(f) = \Var(f \mid \sum_i X_i > 0)$ by symmetry, and similarly observe that the Dirichlet form of the Glauber dynamics on the whole measure is equal to the Dirichlet form of the censored dynamics when restricted to even $f$. Hence, the conclusion follows immediately. 
\end{proof}
This is not enough to bound the higher-order spectral gap since it says nothing about odd functions; we now give a more sophisticated argument which does successfully bound the higher-order spectral gap.
\begin{prop}
Suppose $n$ is odd, $\beta > 1$, and consider the Curie-Weiss model on $\{\pm 1\}^n$,
\[ p(x) \propto \exp\left(\frac{\beta}{2n}\left(\sum_i x_i\right)^2\right). \]
Let $\mathcal L$ be the generator of the continuous-time Glauber dynamics.
Then $\lambda_3(-\mathcal L) = \Omega_{\beta}(1/n^3)$.
\end{prop}
\begin{proof}
Let $f$ be an arbitrary function and observe by the law of total variance that
\[ \Var(f) = \E \Var(f \mid \sgn(\sum_i X_i)) + \Var(\E[f \mid \sgn(\sum_i X_i)]). \]
Therefore,
\[ \Var(f) \le C(\beta) [\mathcal{E}^+(f,f)/2 + \mathcal{E}^-(f,f)/2] + \Var(\E[f \mid \sgn(\sum_i X_i)]) \]
where $\mathcal{E}^+$ is the Dirichlet form of the censored dynamics from \cite{levin2010glauber}, and $\mathcal{E}^-(f,f) = \mathcal{E}^+(x \mapsto f(-x), x \mapsto f(-x))$ is the Dirichlet form of the analogous dynamics on the set of spins with negative magnetization. 

Next, we perform a comparison of Dirichlet forms, i.e. a routing argument, between the censored dynamics and the original Glauber dynamics. See Chapter 13 of \cite{levin2017markov} for more background on this technique. The censored dynamics has additional transitions not present in the Glauber dynamics, which are of the form $x \mapsto -y$ for $x,y$ neighbors such that $\sum_i x_i = 1$ and $\sum_i y_i = -1$. Note that such a pair $(x,-y)$ will differ at all but one coordinate. We route such a transition to a path in the hypercube by replacing the coordinates of $x$ by those of $y$ one-by-one, left-to-right, skipping the one coordinate which is the same. Next, we need to bound the congestion of the path, i.e. how many paths can be routed over the same edge ---- this is at most $2n + 2$, because for a routing path going from $x$ to $-y$ to cross over an edge $(x',x'')$ where $k$ is such that $x'_k \ne x''_k$, it must be that $x$ is at most hamming distance $1$ from the bitstring given by flipping the sign of all the coordinates of $x'$ to the left of $k$, and analogously $-y$ is at most distance $1$ from the bitstring given by flipping the coordinates of $x''$ to the right of $k$. Furthermore, for such a tuple $x,-y,x',x''$ we will have the following inequality of conductances:
\[ \frac{p(x) p(-y)}{p(x) + p(-y)} \le \exp(4\beta) \frac{p(x')p(x'')}{p(x') + p(x'')} \]
because $p(x) = \min_{z \in \{\pm 1\}^n} p(z) \le p(x')$ and because $\frac{p(z)}{p(z) + p(z'')} \in [e^{-2\beta}, e^{2 \beta}]$ for any pair $z,z'$ at hamming distance one, by expanding the definition of $p$.

From this routing, it follows that
\[ \mathcal{E}^+(f,f) \le C'(\beta) n^3 \mathcal{E}(f,f) \]
for some constant $C'(\beta) > 0$ independent of $n$, and likewise for $\mathcal{E}^-$. Thus,
\[ \Var(f) \le C''(\beta) n^3 \mathcal{E}(f,f) + \Var(\E[f \mid \sgn(\sum_i X_i)])\]
which proves the claim by the variational characterization of eigenvalues. 
\end{proof}


As a consequence, we can prove rapid mixing of Glauber dynamics from any initial distribution such that $\E_{\rho} f_2 = 0$. Furthermore, we can show some basic structural properties of the eigenfunction.
\begin{prop}
    Fix $\beta > 1$.
    Require $n$ is odd and consider the Curie-Weiss model on $\{\pm 1\}^n$
\[ p(x) \propto \exp\left(\frac{\beta}{2n}\left(\sum_i x_i\right)^2\right). \]
Let $\mathcal L$ be the generator of the continuous-time Glauber dynamics, and let $f_2$ to be the second eigenvector from the bottom of the spectrum of $\mathcal L$. Then $f_2$ depends on $x$ only through $\sum_i x_i$ and, at least for sufficiently large $n$ with respect to $\beta$, it is odd, i.e. $f_2(-x) = -f_2(x)$. Furthermore, 
the Glauber dynamics initialized from $\rho$ satisfying $\E_{\rho} f_2 = 0$ will achieve total variation distance $\delta > 0$ from stationarity in time $O_{\beta}(n^3(n + \log(1/\delta)))$.
\end{prop}
\begin{proof}
    The mixing statement follows from the previous proposition and our general results. The properties of $f_2$ follow because the dynamics respect the symmetries of permuting coordinates and reversing the role of $+$ and $-$, so applying these symmetries must send the eigenvector $f_2$ to a multiple of itself.  Note that for all $n$ larger than a constant, $f_2$ cannot be an even function (i.e. a positive multiple of itself under interchange of $+$ and $-$) because of \cref{lem:cw-even}, so it must be odd.
\end{proof}
\begin{remark}
We compare this result with what can be achieved using a different argument closely related to the Poincar\'e inequality for even functions.
Using the main result of \cite{levin2010glauber} and making a coupling argument using symmetry, one could show that the Glauber dynamics mix in $O_{\beta}(n \log n)$ time from any \emph{even} initialization, i.e. one such that $\rho(x) = \rho(-x)$. However, this is stricter than the minimal requirement that $\E_{\rho} f_2 = 0$ --- the former corresponds to $2^{n - 1}$ linearly independent constraints, whereas the latter is only a single linear constraint. 
\end{remark}
\subsection{Existence of perfectly balanced initializations with small support}
The following elementary argument tells us that for any finite Markov chain with a higher-order spectral gap, there is an initialization with small support such that the dynamics rapidly mix to stationarity.
\begin{prop}
Suppose that $\mathcal L$ is the generator of a Markov semigroup on a finite state space $\Omega$ with unique stationary measure $\pi$ and let $k \ge 1$ be arbitrary such that $\lambda_k(-\mathcal L) > 0$. Let $f_1,\ldots,f_k$ be the first $k$ eigenfunctions of $\mathcal L$, starting from the bottom of the spectrum.
There exists a distribution $\tilde \rho$ supported on at most $k - 1$ points of $\Omega$, such that $\E_{\tilde \rho}[f_i] = 0$ for all $1 < i < k$, and as a consequence $d_{TV}(\tilde \rho e^{tL}, \pi) \le \delta$ provided that $t \ge (1/\lambda_k(-\mathcal L)) \log(|\Omega|/\delta)$.
\end{prop}
\begin{proof}
    Consider the polytope of distributions $\rho$ satisfying the usual constraints that $\sum_x \rho(x) = 1$, $\rho(x) \ge 0$ for all $x$, along with the additional constraint that $\E_{\rho}[f_i] = 0$ for all $i$ from $2$ to $k - 1$.  This polytope is nonempty because it contains the stationary distribution of the Markov chain. The existence then follows from complementary slackness --- the polytope contains an extreme point $\tilde \rho$, and because of dimension counting, for all but $k - 1$ many points $x$ we must have that $\tilde \rho(x) = 0$.
\end{proof}
\begin{remark}
    The previous result can be generalized to the case where the stationary distribution is not unique---in this case, the expected values of the eigenfunctions with eigenvalue zero should be selected to match the desired stationary measure.  
\end{remark}
\begin{remark}
In some cases it may be possible to explicitly solve for such a distribution. 
For example, in the Curie-Weiss model, the symmetrized all-ones initialization $(1/2) \delta_{\vec 1} + (1/2) \delta_{- \vec 1}$ is an explicit example of such a distribution $\tilde \rho$.
\end{remark}
\printbibliography
\appendix

\section{Consequences of functional inequalities}
\label{s:fi-conseq}

Poincar\'e and log-Sobolev inequalities imply sub-exponential and sub-gaussian concentration of Lipschitz functions, respectively.

\begin{lemma}[Sub-exponential concentration given Poincar\'e inequality, {\cite[Pr. 4.4.2]{bakry2014analysis}}]
\label{l:subexp-p}
    Suppose that $\mu$ satisfies a Poincar\'e inequality with constant $\CP$. Let $f$ be a 1-Lipschitz function.
Then for any $t\in \left[0,\fc{2}{\sqrt{\CP}}\right)$,
\begin{align*}
    \E_\mu e^{tf} &\le 
    \fc{2+t\sqrt{\CP}}{2-t\sqrt{\CP}} e^{t\E_\mu f}.
\end{align*}
\end{lemma}

\begin{lemma}[Herbst, Sub-exponential and sub-gaussian concentration given log-Sobolev inequality, {\cite[Pr. 5.4.1]{bakry2014analysis}}]
\label{l:herbst}
Suppose that $\mu$ satisfies a log-Sobolev inequality with constant $\CLS$. Let $f$ be a 1-Lipschitz function.
Then
\begin{enumerate}
\item
(Sub-exponential concentration) For any $t\in \R$, 
\[
\E_\mu e^{tf} \le e^{t\E_\mu f + \fc{\CLS t^2}2}.
\]
\item
(Sub-gaussian concentration) For any $t\in \left[0,\rc{\CLS}\right)$, 
\[
\E_\mu e^{\fc{tf^2}2} \le
\rc{\sqrt{1-\CLS t}}\exp\ba{\fc{t}{2(1-\CLS t)} (\E_\mu f)^2}.
\]
\end{enumerate}
\end{lemma}

From this, we can get concentration around the mean. First, we note we can bound the variance by the Poincar\'e constant: if $\mu$ satisfies a Poincar\'e inequality with constant $\CP$, 
\begin{align}\label{e:vard}
\E_\mu [\norm{x - \E_\mu x}^2] \le 
\sumo id \Var_\mu(x_i) \le\CP d.
\end{align}

\begin{lemma}
\label{l:pi-conc}
If $\mu$ satisfies a Poincar\'e inequality with constant $\CP$, then
\begin{align*}
    \Pj\pa{\ve{x-\E_\mu x}\ge \sqrt{\CP}(\sqrt d + u) } &\le 3e^{-u}.
\end{align*}
and 
\begin{equation}\label{eq:poincare moment bound}
 \forall p\geq 1: \quad \E_{x\sim \mu}[\ve{ x- \E_\mu x}^p ]^{1/p} = O(p \sqrt{\CP d}) .
\end{equation}
\end{lemma}
\begin{proof}
    By Markov's inequality and \Cref{l:subexp-p}, for $t\in \left[0,\fc{2}{\sqrt{\CP}}\right)$,
    \begin{align*}
    \Pj\pa{\ve{x-\E x}\ge  \sqrt{\CP}(\sqrt d + u)}
        &=\Pj\pa{e^{t\ve{x-\E x}}\ge e^{t\sqrt{\CP}(\sqrt d + u)}}\\
    &\le\fc{\E e^{t\ve{x-\E x}}}{e^{t\sqrt{\CP}(\sqrt d + u)}}
    \le 
    e^{t\E\ve{x-\E x}} \fc{2+t\sqrt{\CP}}{2-t\sqrt{\CP}} e^{-t\sqrt{\CP}(\sqrt d + u)}.
    \end{align*}
    Eq.~\eqref{e:vard} gives $\E\ve{x-\E x}\le (\E \ve{x-\E x}^2)^{1/2}\le \sqrt{\CP d}$. Substituting $t=\rc{\sqrt{\CP}}$ gives the first inequality.
From the above, we also have for $ t= \fc{1}{\sqrt{\CP d}}$
\[\E e^{t\ve{x-\E x}} \leq e^{t \E \ve{x-\E_x}} \fc{2+t\sqrt{\CP}}{2-t\sqrt{\CP}} \leq 3 e\]
thus the moment bound follows from \cite[Proposition 2.7]{vershynin2018high}.
\end{proof}


\begin{lemma}[Upper bound on partition function]
\label{l:Z-bound}
    Suppose that $\pi \propto e^{-f}$ satisfies a Poincar\'e inequality with constant $\CP.$ 
    Then 
    \[
\int_{\R^d} e^{-f(x)}\dx\le \fc{e}{\sqrt{\pi d}} (2e\pi \CP)^{d/2}
    \]
\end{lemma}
\begin{proof}
Without loss of generality, assume $ \E_\pi x =0.$
Consider a density $q$ supported on the ball $B_R(0)$ of radius $R$ centered at $0.$ 
such that $\int_{\R^d} q(x)\dx = \int_{\R^d} e^{-f(x)}\dx$ and $q(x)\equiv 1$ on the ball and $0$ everywhere else. 
Then 
\begin{align*}
   \int_{x\in B_R(0)} (q(x) -e^{-f(x)} )\norm{x}^2\dx &\leq R^2 \int_{x\in B_R(0)} (q(x) -e^{-f(x)} )\dx \\
   &=R^2 \int_{x\not\in B_R(0)}  e^{-f(x)}  \dx \\
   &\leq  \int_{x\not\in B_R(0)}  e^{-f(x)} \norm{x}^2 \dx      
\end{align*}
where the first inequality is due to $ q(x) -e^{-f(x)} \geq 0$ for all $x \in B_R(0)$ and
equality is due to \[\int_{x\in B_R(0)} q(x) \dx=\int_{x\in \R^d} q(x) \dx = \int_{x \in B_R(0)} e^{-f(x)} \dx +  \int_{x \not\in B_R(0)} e^{-f(x)} \dx \]
Hence
\begin{align*}
\fc{\int_{\R^d} e^{-f(x)} \norm{x}^2\dx}{\int_{\R^d} e^{-f(x)}\dx}
\ge \fc{\int_{\R^d} q(x) \norm{x}^2\dx}{\int_{\R^d} q(x)\dx}
= \fc{\int_0^R r^{d-1}r^2\,dr}{\int_0^R r^{d-1}\,dr} = \fc{d}{d+2}R^2.
\end{align*}
This implies $R\le \sqrt{\CP(d+2)}$. Then by the volume of a $n$-ball and Stirling's formula,
\begin{align}
\nonumber
\int_{\R^d} e^{-f(x)}\dx &= 
\int_{\R^d} q(x)\dx = \fc{\pi^{d/2}}{\Ga\pa{\fc d2+1}} \pa{\CP(d+2)}^{d/2}\\
&\le \rc{\sqrt{\pi d}} \pf{2e\pi \CP(d+2)}{d}^{d/2} 
\le \fc{e}{\sqrt{\pi d}} (2e\pi \CP)^{d/2}.
\label{e:e-f-lb}
\end{align}
\end{proof}

\begin{lemma}\label{l:dV-mean}
    Suppose that $\pi \propto e^{-V}$ satisfies a Poincar\'e inequality with constant $\CP$ and $V$ is $\be$-smooth. Then
    \[
\ve{\nb V(\E_\pi x)} \le 2\sqrt{\be} (\sqrt{d} + \sqrt{\be \CP}(\sqrt d + \ln 6))\lesssim \be \sqrt{\CP d}.
    \]
\end{lemma}
\begin{proof}
    By \Cref{l:pi-conc}, $\ve{x-\E_\pi x}\le \sqrt{\CP}(\sqrt d + \ln 6)$ with probability $\ge \rc 2$. Under this event, $ \norm{\nb V(\E_{\pi} x)} \leq \norm{\nb V(x)} + \be \sqrt{\CP}(\sqrt d + \ln 6)$. Hence 
\[\rc 2 \norm{\nb V(\E_\pi x)} \leq \E_{\pi}[\ve{\nb V}] + \be \sqrt{\CP}(\sqrt d + \ln 6)  \leq \sqrt{\beta d}+ \be \sqrt{\CP}(\sqrt d + \ln 6) \]
where the last inequality follows from $ \E_{\pi} [\ve{\nb V}] \leq \E_{\pi}[\ve{\nb V}^2]^{1/2} \leq \sqrt{\beta d}$ by \cref{l:norm bound}.

    The final inequality holds after observing that $\beta C_P\gtrsim 1$, which can be proven (for example) by combining the Cramer-Rao bound (viewing $X$ as an unbiased estimator for the mean of the family of translates of the distribution $\pi$, see e.g.\ \cite{van2000asymptotic}) and the Poincar\'e inequality to show that $1/\beta \lesssim \Var(x_1) \lesssim C_P$.
\end{proof}

\begin{lemma}\label{l:V-mean}
Suppose that $\pi \propto e^{-V}$ satisfies a Poincar\'e inequality with constant $\CP$ and $V$ is $\be$-smooth. Let $R_\ep = \sqrt{\CP} \pa{\sqrt d + \ln \pf{3}{\ep}}$. Then
    \[
V(\E_\pi x) - \min V \le 
R_{1/2} \ve{\nb V(\ol x)} + \fc\be 2 R_{1/2}^2 + \ln 2 - \fc d2 \ln \pf{\be}{2\pi}\lesssim \be \CP d.
    \]
\end{lemma}
\begin{proof}
Let $x^*\in \amin V$ and $\ol x = \E_\pi x$. 
    We upper and lower bound $\int_{\R^d}e^{-V(x)}\dx$.
    For the lower bound, since $V$ is $\be$-smooth, letting $x^*\in \amin V$,
     we have 
     \[V(x) - V(x^*) \leq \langle x-x^*, \nb V(x^*)\rangle + \beta \norm{x-x^*}/2 =\beta \norm{x-x^*}/2 \]
     since $\na V(x^*)=0.$ Thus
    \[
\int_{\R^d} e^{-V(x)}\dx \ge e^{-\min V}\int_{\R^d} \exp(-\beta \norm{x-x^*}/2 ) dx = e^{-\min V}  \pf{\be}{2\pi}^{\fc d2}. 
    \]
    For the upper bound, since by \Cref{l:pi-conc}, $\ve{x-\ol x} \le R_{1/2}$ with probability $\ge \rc 2$, we have
    \begin{align*}
        \int_{\R^d} e^{-V(x)}\dx \le 2 \int_{B_{R_{1/2}}(\ol x)} e^{-V(x)}\dx
        \le 2 e^{-(V(\ol x) - R_{1/2} \ve{\nb V(\ol x)} - \fc \be 2 R_{1/2}^2)},
    \end{align*}
    where the last inequality follows from 
    \begin{align*} V(x)-V(\bar{x}) \geq \langle \nb V(\ol x), x-\ol x\rangle -\frac{\beta}{2} \norm{x-\ol x}^2 &\geq \norm{\nb V(\ol x)} \norm{x-\ol x} - \frac{\beta}{2} \norm{x-\bar{x}}^2 \\
    &\geq \norm{\nb V(\ol x)}  R_{1/2} - \frac{\beta}{2} R_{1/2}^2 .
    \end{align*}
    Putting these inequalities together and taking the logarithm gives the result.
\end{proof}

\section{Inequalities for mixture distributions}

It is well-known that the log-Sobolev inequality and hypercontractivity inequalities are equivalent, see e.g.\ \cite{van2014probability}. The following lemma gives a weaker version of hypercontractivity which is valid for mixtures of distributions satisfying the log-Sobolev inequality, which depends on the minimum weight in the mixture. 
\begin{lemma}[Hypercontractivity for mixtures, {\cite[Lemma 26]{lee2024convergence}}]\label{l:hyper} Let $P_t$ be a reversible Markov process with stationary distribution $\pi = \sum_k w_k \pi_k$. Let 
$q(t) = 1+(p-1)e^{2t/C^*}$ where $C^*=\max_k c_k$.
Assume that the following hold. 
\begin{enumerate}
    \item There exists a decomposition of the form, 
    $$\langle f, \mathscr{L}f \rangle_\pi \leq \sum_{k=1}^m w_i \langle f, \mathscr{L}_k f \rangle_{\pi_k}.$$
    \item For each $\pi_k$ there exists a log-Sobolev inequality of the form, 
    $$\Ent_{\pi_k}[f^2] \leq 2c_k\cdot\mathscr{E}_{\pi_k}(f,f).$$
\end{enumerate}

Let $f>0$ and $w^*=\min_k w_k$.
Then $\fc{\norm{P_tf}_{L^{q(t)}(\pi)}}{(w^*)^{\rc{q(t)}}}$ is a non-increasing function in $t$:
\begin{equation*}
    \frac{\Vert P_tf\Vert _{L^{q(t)}(\pi)}}{(w^*)^\frac{1}{q(t)}} \leq  \frac{\Vert P_0f\Vert _{L^{q(0)}(\pi)}}{(w^*)^\frac{1}{q(0)}} = \frac{\Vert f\Vert _{L^p(\pi)}}{(w^*)^\frac{1}{p}} 
\end{equation*}
and
\[
\Vert P_t f\Vert_{L^{q(t)}(\pi)} \le \theta(q(t),p) \Vert f \Vert_{L^p(\pi)}
\]
where
$\theta(q,p) = \pf{1}{w^*}^{\frac{1}{p}-\frac{1}{q}}.$

\end{lemma}

The following two lemmas concern smoothness properties of the mixture distribution given smoothness of the components.

\begin{lemma}[Hessian bound for mixture]\label{l:hessian bound for mixture}
    Suppose $\pi=\exp(-V(x))=\sum_{i=1}^k p_i \pi_i$ where each $\pi_i$ is $\beta$-smooth. 
    Then
    \[-(\be + G(x)^2 ) I \preceq \nabla^2 V(x)\preceq \beta I\]
    where $G(x)=\max_i \ve{\nb V_i(x)}.$
    

\end{lemma}
\begin{proof}
We note that
\begin{equation} \label{eq:hessian}
  \nabla V^2(x) = \frac{\sum_{i=1}^k p_i\pi_i(x) \nabla^2 V_i(x) }{\pi(x)} - \frac{\sum_{i=1}^k p_i\pi_i(x) (\nabla V_i (x) -\nabla V_j(x)) (\nabla V_i (x) -\nabla V_j(x))^\top }{4\pi^2(x)}  .
\end{equation}
The claim follows from $ -\beta I\preceq \nabla^2 V_i(x) \preceq \beta I$ and \[ 0\preceq  (\nabla V_i (x) -\nabla V_j(x)) (\nabla V_i (x) -\nabla V_j(x))^\top\preceq \ve{\nabla V_i (x) -\nabla V_j(x)}^2 I \preceq 4 G(x)^2 I.\qedhere\]
\end{proof}
\begin{lemma} \label{l:norm bound}
    If $\pi=\exp(-V(x))=\sum_{i=1}^k p_i \pi_i$ where each $\pi_i$ is $\beta$-smooth, then $ \E_{\pi} [\ve{\nb V (x)}^2]\leq \beta d.$
\end{lemma}
\begin{proof}
    This follows from \cite[Lemma 16]{chewi2021analysis} i.e. 
    \[\E_\pi[\ve {\nb V(x)}^2] = \E_\pi[ \Delta V(x)]\] and noting that
    \[ \Delta V(x) = \Trace (\nabla^2 V(x)) \leq \Trace \pa{\frac{\sum_i p_i \pi_i(x) \nabla V_i^2(x)}{\pi(x)}} =\frac{\sum_i p_i \pi_i(x) \Trace (\nb^2 V_i(x)) }{\pi(x)}\leq \beta d\]  
where the first inequality is due to \[\nb^2 V  \preceq \frac{\sum_i p_i \pi_i (x)\nb^2 V_i(x)}{\pi(x)}, \] which is implied by \Cref{eq:hessian} and the last inequality is due to smoothness of each $\pi_i.$
\end{proof}
\end{document}